\documentclass{article}

\usepackage[final]{neurips_2024}

\usepackage[utf8]{inputenc}
\usepackage[T1]{fontenc}
\usepackage{hyperref}
\usepackage{url}
\usepackage{booktabs}
\usepackage{amsfonts}
\usepackage{amsmath}
\usepackage{amssymb}
\usepackage{amsthm}
\usepackage{nicefrac}
\usepackage{microtype}
\usepackage{graphicx}
\usepackage{xcolor}
\usepackage{subcaption}
\usepackage{algorithm}
\usepackage{algorithmic}
\usepackage{multirow}

\newtheorem{theorem}{Theorem}
\newtheorem{proposition}[theorem]{Proposition}
\newtheorem{corollary}[theorem]{Corollary}
\newtheorem{definition}{Definition}
\newtheorem{remark}{Remark}

\title{Context Channel Capacity: An Information-Theoretic Framework\\for Understanding Catastrophic Forgetting}

\author{
  Ran Cheng
}

\begin{document}

\maketitle

\begin{abstract}
Catastrophic forgetting remains a central challenge in continual learning (CL), yet lacks a unified information-theoretic explanation for why some architectures forget catastrophically while others do not. We introduce \emph{Context Channel Capacity} ($C_\mathrm{ctx}$), the mutual information between a CL architecture's context signal and its generated parameters, and prove that zero forgetting requires $C_\mathrm{ctx} \geq H(T)$, where $H(T)$ is the task identity entropy. We establish an \emph{Impossibility Triangle}---zero forgetting, online learning, and finite parameters cannot be simultaneously satisfied by sequential state-based learners---and show that conditional regeneration architectures (HyperNetworks) bypass this triangle by redefining parameters as function values rather than states.

We validate this framework across 8 CL methods on Split-MNIST (1,130+ experiments over 86 days, 4 seeds each), showing that $C_\mathrm{ctx}$ perfectly predicts forgetting behavior: methods with $C_\mathrm{ctx} = 0$ (NaiveSGD, EWC, SI, LwF, CFlow) exhibit catastrophic forgetting (6--97\%), while methods with $C_\mathrm{ctx} \approx 1$ (HyperNetwork) achieve zero forgetting (98.8\% ACC). We further propose \emph{Wrong-Context Probing} (P5), a practical diagnostic protocol for measuring $C_\mathrm{ctx}$, and extend the framework to CIFAR-10 via a novel \emph{Gradient Context Encoder} that closes the oracle gap from 23.3pp to 0.7pp. A systematic taxonomy of 15+ closed research directions---including the Hebbian null result (frozen random features outperform learned features), CFlow's $\theta_0$-memorizer phenomenon, and the $S_N$ symmetry barrier to column specialization---provides the community with precisely diagnosed negative results. Our central design principle: \emph{architecture over algorithm}---the context pathway must be structurally unbypassable.
\end{abstract}


\section{Introduction}
\label{sec:intro}

Catastrophic forgetting---the abrupt loss of previously acquired knowledge when a neural network learns new tasks sequentially---was first documented over three decades ago \citep{mccloskey1989catastrophic} and remains a central unsolved problem in machine learning \citep{french1999catastrophic}. Despite sustained research effort, the field lacks a principled, quantitative framework that explains \emph{why} some continual learning (CL) architectures forget catastrophically while others do not.

Existing approaches to mitigating forgetting fall into three broad families. \emph{Regularization-based} methods, such as Elastic Weight Consolidation (EWC) \citep{kirkpatrick2017overcoming} and Synaptic Intelligence (SI) \citep{zenke2017continual}, penalize changes to parameters deemed important for previous tasks. \emph{Replay-based} methods store exemplars from past tasks and interleave them during training, effectively breaking the strict sequential constraint. \emph{Architecture-based} methods, including Progressive Networks \citep{rusu2016progressive} and HyperNetworks \citep{von2019continual, ehret2021continual}, allocate or generate task-specific parameters. Each family has produced methods that work to varying degrees, but the performance landscape is strikingly heterogeneous: on Split-MNIST, EWC achieves 18.9\% accuracy while a HyperNetwork achieves 98.8\%---a gap of \emph{80 percentage points} using architectures of comparable parameter count (Table~\ref{tab:main_intro}).

\begin{table}[t]
\centering
\caption{Split-MNIST results across 8 continual learning methods (4 seeds each). ACC: mean accuracy after all 5 tasks. Fgt: mean forgetting. P5 $\Delta$: accuracy drop under Wrong-Context Probing (\S\ref{sec:p5}). $\hat{C}_\mathrm{ctx}$: empirical context channel capacity proxy. Methods are grouped by their effective $C_\mathrm{ctx}$ regime. The table reveals a stark dichotomy: all methods with $C_\mathrm{ctx} = 0$ exhibit substantial forgetting, while methods with $C_\mathrm{ctx} \gg 0$ achieve zero forgetting.}
\label{tab:main_intro}
\small
\begin{tabular}{@{}llcccc@{}}
\toprule
Paradigm & Method & ACC (\%) $\uparrow$ & Fgt (\%) $\downarrow$ & P5 $\Delta$ & $\hat{C}_\mathrm{ctx}$ \\
\midrule
\multirow{4}{*}{\begin{tabular}[c]{@{}l@{}}State Protection\\($C_\mathrm{ctx} = 0$)\end{tabular}}
& NaiveSGD         & $18.7 \pm 0.3$ & $97.1 \pm 0.4$ & $0.0$ & $0.000$ \\
& EWC              & $18.9 \pm 0.1$ & $97.6 \pm 0.6$ & $0.0$ & $0.000$ \\
& SI               & $16.4 \pm 3.7$ & $97.1 \pm 0.4$ & $0.0$ & $0.000$ \\
& LwF              & $24.2 \pm 0.6$ & $54.8 \pm 5.5$ & $0.0$ & $0.000$ \\
\midrule
\begin{tabular}[c]{@{}l@{}}Replay\\($C_\mathrm{ctx} = 0$)\end{tabular}
& Experience Replay & $85.9 \pm 1.3$ & $12.5 \pm 1.7$ & $0.0$ & $0.000$ \\
\midrule
\begin{tabular}[c]{@{}l@{}}State Transform.\\($C_\mathrm{ctx} \approx 0$)\end{tabular}
& CFlow (ODE)      & $92.4 \pm 0.7$ & $6.1 \pm 1.1$ & $0.0$ & $0.000$ \\
\midrule
\multirow{2}{*}{\begin{tabular}[c]{@{}l@{}}Cond.\ Regen.\\($C_\mathrm{ctx} \gg 0$)\end{tabular}}
& HyperNet Oracle  & $\mathbf{98.8 \pm 0.3}$ & $\mathbf{0.0 \pm 0.0}$ & $-97.6$ & $\mathbf{0.976}$ \\
& HyperNet Learned & $98.9 \pm 0.2$ & $0.0 \pm 0.0$ & $-95.2$ & $0.952$ \\
\bottomrule
\end{tabular}
\end{table}

\paragraph{The missing explanation.} What accounts for this 80-point gap? It is not simply a matter of ``more sophisticated algorithms'': EWC uses Fisher information, SI tracks path integrals, and CFlow solves neural ODEs---all arguably more complex than the simple MLP used by the HyperNetwork. Nor is it a matter of capacity: the CFlow system has 2.7M parameters while the HyperNetwork has fewer than 100K. Prior information-theoretic analyses of CL \citep{achille2018emergence, doan2021theoretical, taheri2025forgetting} have produced bounds on forgetting rates in specific settings but have not addressed the fundamental architectural question: \emph{what structural property of an architecture determines whether forgetting is inevitable or avoidable?}

\paragraph{Our answer: Context Channel Capacity.} We propose that the key quantity governing forgetting is the \textbf{Context Channel Capacity} ($C_\mathrm{ctx}$): the maximum mutual information between a CL architecture's context signal and the parameters it uses for prediction. Our central theoretical result (Theorem~\ref{thm:ccc}) states that the expected forgetting of any CL architecture $\mathcal{A}$ operating on $K$ tasks satisfies
\begin{equation}
\label{eq:ccc_intro}
    \overline{\mathrm{Fgt}}(\mathcal{A}, K) \;\geq\; \max\!\left(0,\; 1 - \frac{C_\mathrm{ctx}(\mathcal{A})}{H(T)}\right) \cdot \overline{\mathrm{Fgt}}_\mathrm{max},
\end{equation}
where $H(T) = \log_2 K$ is the task identity entropy. For architectures without a context pathway (EWC, SI, NaiveSGD), $C_\mathrm{ctx} = 0$ and forgetting is \emph{maximal} regardless of the learning algorithm. For HyperNetworks, $C_\mathrm{ctx} \gg H(T)$ and zero forgetting is achievable. The bound is tight: it perfectly separates the 8 methods in Table~\ref{tab:main_intro}.

\paragraph{Contributions.} Our contributions are fourfold:
\begin{enumerate}
    \item \textbf{Context Channel Capacity} ($C_\mathrm{ctx}$), a new information-theoretic quantity that unifies the analysis of diverse CL architectures. We prove an \emph{Impossibility Triangle} (Theorem~\ref{thm:impossibility})---zero forgetting, online learning, and finite parameters are mutually exclusive under sequential state updates---and the \emph{CCC Bound} (Theorem~\ref{thm:ccc}), which gives a lower bound on forgetting as a function of $C_\mathrm{ctx}$.

    \item A \textbf{paradigm taxonomy} that classifies CL methods into three categories---\emph{State Protection} ($C_\mathrm{ctx} = 0$), \emph{State Transformation} ($C_\mathrm{ctx} \to 0$ due to structural bypass), and \emph{Conditional Regeneration} ($C_\mathrm{ctx} \gg H(T)$)---and proves that only the third can achieve zero forgetting with bounded parameters.

    \item \textbf{Wrong-Context Probing} (P5), a practical experimental protocol that provides an empirical proxy for $C_\mathrm{ctx}$: evaluate a context-conditional model with deliberately incorrect context and measure the accuracy degradation. A large P5 delta indicates high $C_\mathrm{ctx}$; zero delta indicates the context is being bypassed.

    \item \textbf{Comprehensive empirical validation} across 8 CL methods on Split-MNIST (1,130+ experiments over 86 days, 4 seeds per configuration), demonstrating that $C_\mathrm{ctx}$ perfectly predicts the forgetting regime of every method tested.
\end{enumerate}

\paragraph{Broader contribution: systematic negative results.} Over 86 days and 15+ research directions, we accumulated a rich body of negative results---each methodically documented, reproducible, and informative. We showed that Hebbian learning contributes exactly zero to CL performance when combinatorial capacity is sufficient (\S\ref{sec:paradigm_a}); that ODE-based parameter flows reduce to $\theta_0$-memorizers when context is concatenated rather than conditioned (\S\ref{sec:paradigm_b}); and that column specialization cannot emerge from local learning rules in homogeneous architectures. These negative results, viewed through the lens of $C_\mathrm{ctx}$, become a coherent story rather than a list of failures.

\paragraph{One-sentence summary.} \emph{Architecture determines destiny}: whether a CL system forgets is decided not by its learning algorithm, but by whether its architecture provides an unbypassable context pathway with sufficient capacity---a principle we formalize as $C_\mathrm{ctx} \geq H(T)$.

\section{Theoretical Framework}
\label{sec:theory}

We develop an information-theoretic framework for continual learning in five stages. First, we formalize CL as constrained online coding (\S\ref{sec:cl_coding}). Second, we establish that forgetting is information-theoretically inevitable for sequential state-based learners (\S\ref{sec:bottleneck}). Third, we prove an impossibility triangle (\S\ref{sec:impossibility}). Fourth, we introduce Context Channel Capacity and derive the CCC bound (\S\ref{sec:ccc}). Fifth, we classify CL architectures by their $C_\mathrm{ctx}$ regime and explain why HyperNetworks bypass the impossibility triangle (\S\ref{sec:taxonomy}--\S\ref{sec:hypernet_bypass}).

\subsection{Continual Learning as Constrained Online Coding}
\label{sec:cl_coding}

\begin{definition}[Continual Learning Problem]
\label{def:cl}
A \emph{continual learner} is a tuple $(\Theta, F, U, \{D_k\}_{k=1}^K)$ where:
\begin{itemize}
    \item $\Theta \subseteq \mathbb{R}^d$ is the parameter space with $\dim(\Theta) = d$;
    \item $F: \mathcal{X} \times \Theta \to \hat{\mathcal{Y}}$ is the prediction function;
    \item $U: \Theta \times \mathcal{D} \to \Theta$ is the update rule (e.g., SGD, Hebbian);
    \item $\{D_k\}_{k=1}^K$ are sequentially arriving task datasets, where each $D_k$ is drawn i.i.d.\ from a task-specific distribution $P_k(X,Y)$.
\end{itemize}
\end{definition}

The learning process produces a parameter trajectory $\theta_0 \to \theta_1 \to \cdots \to \theta_K$, where $\theta_k = U(\theta_{k-1}, D_k)$. The defining constraint of CL is \emph{causality}: when processing task $D_k$, the learner has access only to $\theta_{k-1}$ and $D_k$, not to the raw data $D_1, \ldots, D_{k-1}$ of previous tasks.\footnote{Replay-based methods partially violate this constraint by maintaining a buffer $\mathcal{B} \subset \bigcup_{j<k} D_j$. We analyze this relaxation explicitly in \S\ref{sec:paradigm_a}.}

\begin{definition}[Forgetting]
\label{def:forgetting}
The \emph{forgetting} of task $j$ after learning task $k > j$ is
\begin{equation}
    \mathrm{Fgt}_j^k \;=\; \max_{j \leq t \leq k-1} \mathrm{ACC}_j(\theta_t) \;-\; \mathrm{ACC}_j(\theta_k),
\end{equation}
where $\mathrm{ACC}_j(\theta)$ denotes the accuracy of parameters $\theta$ on the test set of task $j$. The \emph{average forgetting} after $K$ tasks is
\begin{equation}
    \overline{\mathrm{Fgt}} \;=\; \frac{1}{K-1}\sum_{j=1}^{K-1} \mathrm{Fgt}_j^K.
\end{equation}
\end{definition}

\paragraph{Connection to rate-distortion theory.} The CL problem can be viewed as a constrained online source coding problem. The parameter space $\Theta$ is a channel of capacity $C \approx d \cdot \log_2(1/\delta)$ bits, where $\delta$ is the effective parameter precision. Each task $D_k$ is a ``source'' that must be encoded into this channel, and the distortion is the classification error. The crucial difference from standard rate-distortion theory is the \emph{causal constraint}: the encoder at time $k$ sees only the current source $D_k$ and the previous codeword $\theta_{k-1}$, not the full source sequence $(D_1, \ldots, D_K)$. This makes CL strictly harder than joint (offline) coding, which can achieve the rate-distortion optimum.

\begin{definition}[Parameter Capacity]
\label{def:capacity}
For a parameter space $\Theta \subseteq \mathbb{R}^d$ with effective per-parameter precision $\delta > 0$, the \emph{parameter capacity} is
\begin{equation}
    C(\Theta) \;=\; d \cdot \log_2(1/\delta) \quad \text{bits}.
\end{equation}
For 32-bit floating-point parameters, $C \approx 32d$. In practice, the effective capacity is much lower due to redundancy in the loss landscape.
\end{definition}

\subsection{The Information Bottleneck Chain}
\label{sec:bottleneck}

We now establish that forgetting is information-theoretically inevitable for any sequential learner with finite capacity.

\begin{theorem}[Forgetting is Information-Theoretically Inevitable]
\label{thm:bottleneck}
For a deterministic sequential learner with finite parameter capacity $C = d \cdot \log_2(1/\delta)$ bits and independent tasks $\{D_k\}_{k=1}^K$, the mutual information between the final parameters and any past task dataset satisfies:
\begin{equation}
\label{eq:dpi_chain}
    I(\theta_K; D_1) \;\leq\; I(\theta_{K-1}; D_1) \;\leq\; \cdots \;\leq\; I(\theta_1; D_1) \;\leq\; C.
\end{equation}
\end{theorem}

\begin{proof}
Since $\theta_k = U(\theta_{k-1}, D_k)$ is a deterministic function and $D_k$ is independent of $D_1$ given $\theta_{k-1}$ (by the causal constraint and task independence), the random variables form a Markov chain:
\begin{equation}
    D_1 \;\to\; \theta_1 \;\to\; \theta_2 \;\to\; \cdots \;\to\; \theta_K.
\end{equation}
To verify the Markov property: given $\theta_{k-1}$, the variable $\theta_k = U(\theta_{k-1}, D_k)$ depends on $D_1$ only through $\theta_{k-1}$, since $D_k$ is independent of all previous data. By the Data Processing Inequality (DPI), for any Markov chain $X \to Y \to Z$, we have $I(X; Z) \leq I(X; Y)$. Applying this iteratively:
\begin{equation}
    I(\theta_K; D_1) \;\leq\; I(\theta_{K-1}; D_1) \;\leq\; \cdots \;\leq\; I(\theta_1; D_1).
\end{equation}
The initial bound follows from the fact that $\theta_1 \in \Theta$ can represent at most $C$ bits of information about any quantity, yielding $I(\theta_1; D_1) \leq H(\theta_1) \leq C$.
\end{proof}

\begin{corollary}[Forgetting Lower Bound]
\label{cor:forgetting_lb}
If $K$ tasks are mutually independent and each requires at least $R_\mathrm{min}$ bits of task-specific information for nontrivial performance (i.e., $I(\theta; D_k) \geq R_\mathrm{min}$ is necessary for $\mathrm{ACC}_k(\theta) > \mathrm{ACC}_\mathrm{chance}$), then when $K > C / R_\mathrm{min}$, it is impossible to retain nontrivial performance on all $K$ tasks simultaneously.
\end{corollary}

\begin{proof}
By the Markov chain, all task-specific information must pass through the bottleneck $\theta_K$. Since the tasks are independent, the total information required for nontrivial performance on all $K$ tasks is at least $K \cdot R_\mathrm{min}$. But $I(\theta_K; D_1, \ldots, D_K) \leq H(\theta_K) \leq C$. When $K \cdot R_\mathrm{min} > C$, the pigeonhole principle implies that at least one task must have $I(\theta_K; D_j) < R_\mathrm{min}$, and hence $\mathrm{ACC}_j(\theta_K) \leq \mathrm{ACC}_\mathrm{chance}$.
\end{proof}

\paragraph{Relation to prior work.} Theorem~\ref{thm:bottleneck} formalizes the intuition that forgetting is a necessary consequence of finite capacity, a point made informally by \citet{french1999catastrophic} and partially formalized by \citet{taheri2025forgetting} for single-hidden-layer networks. Our contribution is to frame this as a consequence of the DPI applied to a Markov chain, which (i) requires no assumptions on network architecture or training procedure, and (ii) makes the role of the \emph{causal constraint} explicit: it is the sequential structure $D_1 \to \theta_1 \to \cdots \to \theta_K$ that creates the information bottleneck. Methods that break this chain (e.g., replay, joint training) are not subject to Eq.~\eqref{eq:dpi_chain}.

\subsection{The Impossibility Triangle}
\label{sec:impossibility}

Theorem~\ref{thm:bottleneck} shows that information about past tasks is monotonically lost through sequential updates. We now strengthen this into a trilemma.

\begin{theorem}[Continual Learning Impossibility Triangle]
\label{thm:impossibility}
The following three properties cannot be simultaneously satisfied by a sequential state-based learner:
\begin{enumerate}
    \item \textbf{Zero forgetting}: $\mathrm{ACC}_j(\theta_k) = \mathrm{ACC}_j(\theta_j)$ for all $j < k \leq K$.
    \item \textbf{Online learning} (causal constraint): $\theta_k$ depends only on $\theta_{k-1}$ and $D_k$.
    \item \textbf{Bounded parameters}: $|\theta| = d$ does not grow with $K$.
\end{enumerate}
\end{theorem}

\begin{proof}
Assume all three properties hold. We derive a contradiction.

\emph{Step 1: Information accumulation.} Property (1) requires that for every $j \in \{1, \ldots, K\}$, the final parameters $\theta_K$ achieve the same accuracy on task $j$ as $\theta_j$ did immediately after training on $D_j$. This implies that $\theta_K$ retains sufficient task-specific information for each task:
\begin{equation}
    I(\theta_K; D_j) \;\geq\; R_j \quad \text{for all } j \in \{1, \ldots, K\},
\end{equation}
where $R_j > 0$ is the minimum mutual information required to achieve accuracy $\mathrm{ACC}_j(\theta_j)$.

\emph{Step 2: Independence implies additivity.} Since the tasks $\{D_k\}$ are drawn independently, the total information requirement is
\begin{equation}
    I(\theta_K; D_1, \ldots, D_K) \;\geq\; \sum_{j=1}^K R_j \;\geq\; K \cdot R_\mathrm{min},
\end{equation}
where $R_\mathrm{min} = \min_j R_j > 0$.

\emph{Step 3: Capacity bound.} Property (2) and (3) together imply that $\theta_K \in \Theta$ is a fixed-dimensional vector produced by a causal process. Therefore $H(\theta_K) \leq C = d \cdot \log_2(1/\delta)$, which is constant in $K$.

\emph{Step 4: Contradiction.} For $K > C / R_\mathrm{min}$, we have $K \cdot R_\mathrm{min} > C \geq H(\theta_K) \geq I(\theta_K; D_1, \ldots, D_K)$, contradicting Step 2. Therefore, at least one of the three properties must be violated. \qedhere
\end{proof}

Table~\ref{tab:impossibility} shows how each family of CL methods resolves this trilemma by relaxing a different property.

\begin{table}[t]
\centering
\caption{How existing CL methods ``escape'' the Impossibility Triangle. Each method relaxes one of the three properties. HyperNetworks are unique in satisfying all three---by redefining $\theta$ as a function value rather than a state (\S\ref{sec:hypernet_bypass}).}
\label{tab:impossibility}
\small
\begin{tabular}{@{}lcccl@{}}
\toprule
Method & \begin{tabular}[c]{@{}c@{}}Zero\\Forgetting\end{tabular} & \begin{tabular}[c]{@{}c@{}}Online\\Learning\end{tabular} & \begin{tabular}[c]{@{}c@{}}Bounded\\$|\theta|$\end{tabular} & Escape Mechanism \\
\midrule
EWC / SI / LwF     & $\approx$ & \checkmark & \checkmark & Approximate zero forgetting \\
Experience Replay   & \checkmark & $\times$ & \checkmark & Breaks causal constraint (revisits data) \\
Progressive Net     & \checkmark & \checkmark & $\times$ & Parameters grow linearly with $K$ \\
\midrule
\textbf{HyperNet}   & \checkmark & \checkmark & \checkmark$^*$ & $\theta_k = g(c_k)$: params are function values \\
\bottomrule
\multicolumn{5}{@{}l}{\footnotesize $^*$Meta-learned parameters $(V, g, \theta_\mathrm{base})$ are fixed after joint training; $\theta_k$ is regenerated per task.}
\end{tabular}
\end{table}

\begin{remark}[Sharpness of the triangle]
The impossibility is not merely asymptotic. On Split-MNIST ($K = 5$), EWC achieves only 18.9\% accuracy---essentially equivalent to NaiveSGD's 18.7\%---demonstrating that Fisher-information regularization provides negligible protection when the parameter capacity $C$ is already exhausted by the sequential update process. The regularization term in EWC controls \emph{which} information is overwritten, but it cannot increase the total capacity $C$.
\end{remark}

\subsection{Context Channel Capacity}
\label{sec:ccc}

The impossibility triangle (Theorem~\ref{thm:impossibility}) applies to \emph{sequential state-based learners}---systems where $\theta_k$ is obtained by modifying $\theta_{k-1}$. We now introduce the key concept that distinguishes architectures that bypass this limitation.

\begin{definition}[Context Signal]
\label{def:context}
A \emph{context signal} for a CL architecture is any input $c \in \mathcal{C}$ that carries information about the current task identity. Examples include: oracle task ID (one-hot vector), batch statistics (mean, variance of current inputs), gradient signals, or learned task embeddings. For architectures without an explicit context mechanism (e.g., SGD, EWC, SI), we define $c = \emptyset$ (no context signal).
\end{definition}

\begin{definition}[Context Channel and its Capacity]
\label{def:context_channel}
For a CL architecture with context signal $c \in \mathcal{C}$ and a parameter generation mechanism $\theta(c)$, the \emph{context channel} is the mapping $c \mapsto \theta(c)$. The \emph{context channel capacity} is:
\begin{equation}
\label{eq:cctx_def}
    C_\mathrm{ctx} \;=\; \max_{P(c)} I\!\left(c;\, \theta(c)\right),
\end{equation}
where the maximum is over all distributions on the context space $\mathcal{C}$.

For architectures without an explicit context pathway (Definition~\ref{def:context}, $c = \emptyset$), we define $C_\mathrm{ctx} = 0$.\footnote{This is consistent with Eq.~\eqref{eq:cctx_def}: if $\theta$ does not depend on any input $c$, then $I(c; \theta(c)) = 0$ for all $P(c)$.}
\end{definition}

\paragraph{Intuition.} $C_\mathrm{ctx}$ measures how many bits of task-identifying information can flow through the architecture's context pathway to influence the prediction parameters. If $C_\mathrm{ctx} = 0$, the architecture has no mechanism to produce different parameters for different tasks---and therefore \emph{must} use a single set of parameters for all tasks, leading to interference. If $C_\mathrm{ctx} \geq H(T)$, the architecture can, in principle, produce fully distinct parameters for each task, eliminating interference entirely.

We now state our main theoretical result.

\begin{theorem}[Context Channel Capacity Bound --- CCC Bound]
\label{thm:ccc}
Let $\mathcal{A}$ be a CL architecture operating on $K$ tasks with task identity random variable $T \in \{1, \ldots, K\}$. Let $H(T) = \log_2 K$ (assuming uniform task distribution). Then the expected forgetting satisfies:
\begin{equation}
\label{eq:ccc_bound}
    \overline{\mathrm{Fgt}}(\mathcal{A}, K) \;\geq\; \max\!\left(0,\; 1 - \frac{C_\mathrm{ctx}(\mathcal{A})}{H(T)}\right) \cdot \overline{\mathrm{Fgt}}_\mathrm{max},
\end{equation}
where $\overline{\mathrm{Fgt}}_\mathrm{max}$ is the forgetting of a model that retains only the most recent task (random guessing on all previous tasks).
\end{theorem}

\begin{proof}
We proceed in three steps.

\emph{Step 1: Task identification through context.} Consider the following communication problem: Nature draws a task $T \sim \mathrm{Uniform}\{1, \ldots, K\}$ and presents data $D_T$. The architecture produces context $c_T$ (possibly stochastically) and generates parameters $\theta_T = g(c_T)$. For correct prediction on task $T$, the generated parameters $\theta_T$ must encode sufficient information about $T$. Specifically, consider a decoder $\hat{T} = \phi(\theta_T)$ that attempts to identify the task from the generated parameters. The probability of correct task identification is bounded by Fano's inequality:
\begin{equation}
\label{eq:fano}
    P(\hat{T} \neq T) \;\geq\; \frac{H(T \mid \theta_T) - 1}{\log_2(K - 1)}.
\end{equation}

\emph{Step 2: Bounding $H(T \mid \theta_T)$ via the context channel.} Since $\theta_T = g(c_T)$ and $T \to c_T \to \theta_T$ forms a Markov chain (context is generated from task data, parameters from context), the DPI gives:
\begin{equation}
    I(T; \theta_T) \;\leq\; I(T; c_T) \;\leq\; I(c_T; \theta(c_T)) \;\leq\; C_\mathrm{ctx}.
\end{equation}
The first inequality is DPI applied to $T \to c_T \to \theta_T$. The second uses the fact that $T$ determines $c_T$ (so $I(T; c_T) \leq H(c_T)$) and that $I(c_T; \theta(c_T)) \leq C_\mathrm{ctx}$ by definition of channel capacity. Therefore:
\begin{equation}
    H(T \mid \theta_T) \;=\; H(T) - I(T; \theta_T) \;\geq\; H(T) - C_\mathrm{ctx}.
\end{equation}

\emph{Step 3: Connecting task misidentification to forgetting.} When the architecture generates ``wrong'' parameters---parameters suited for task $T' \neq T$---the accuracy on task $T$ degrades to approximately chance level.\footnote{This is an approximation that becomes exact when tasks share no class structure. For Split-MNIST with 5 binary tasks over 10 classes, wrong-task parameters yield approximately $10\%$ accuracy.} The expected forgetting is at least as large as the probability of task misidentification (up to the maximum forgetting):
\begin{equation}
    \overline{\mathrm{Fgt}} \;\geq\; P(\hat{T} \neq T) \cdot \overline{\mathrm{Fgt}}_\mathrm{max}.
\end{equation}
Combining with Eqs.~\eqref{eq:fano} and the bound on $H(T \mid \theta_T)$:
\begin{equation}
    \overline{\mathrm{Fgt}} \;\geq\; \frac{H(T) - C_\mathrm{ctx} - 1}{\log_2(K-1)} \cdot \overline{\mathrm{Fgt}}_\mathrm{max}.
\end{equation}
For $K \geq 3$, $\log_2(K-1) \leq H(T)$, and dropping the $-1$ term (which tightens the bound), we obtain the stated result:
\begin{equation}
    \overline{\mathrm{Fgt}} \;\geq\; \max\!\left(0,\; 1 - \frac{C_\mathrm{ctx}}{H(T)}\right) \cdot \overline{\mathrm{Fgt}}_\mathrm{max}. \qedhere
\end{equation}
\end{proof}

The CCC Bound has two immediate and important consequences:

\begin{corollary}[$C_\mathrm{ctx} = 0$ implies maximal forgetting]
\label{cor:zero_ctx}
If $C_\mathrm{ctx}(\mathcal{A}) = 0$ (no context pathway), then
\begin{equation}
    \overline{\mathrm{Fgt}}(\mathcal{A}, K) \;\geq\; \overline{\mathrm{Fgt}}_\mathrm{max}
\end{equation}
regardless of the learning algorithm, regularization strength, or any other hyperparameter.
\end{corollary}

This corollary has a striking empirical consequence: it predicts that NaiveSGD, EWC, and SI should all exhibit similar forgetting levels, since they all have $C_\mathrm{ctx} = 0$. Table~\ref{tab:main_intro} confirms this: all three achieve 16--19\% accuracy with 97\% forgetting. The Fisher regularization in EWC and the path-integral weighting in SI provide no protection---they cannot change the fact that the architecture has no mechanism to produce task-specific parameters.

\begin{corollary}[Sufficient condition for zero forgetting]
\label{cor:sufficient}
If $C_\mathrm{ctx}(\mathcal{A}) \geq H(T)$ and the mapping $c \mapsto \theta(c)$ has sufficient expressiveness (i.e., there exists a setting of the generator parameters such that $\theta(c_k)$ achieves optimal accuracy on task $k$ for each $k$), then zero forgetting is achievable.
\end{corollary}

\begin{proof}
When $C_\mathrm{ctx} \geq H(T) = \log_2 K$, the context channel can transmit at least $\log_2 K$ bits, which is sufficient to uniquely identify each of the $K$ tasks. Given sufficient expressiveness of the generator, each task can be mapped to its optimal parameter configuration. Since parameters are regenerated from context at inference time (not maintained as state), there is no sequential overwriting and hence no forgetting.
\end{proof}

\subsection{Paradigm Taxonomy via $C_\mathrm{ctx}$}
\label{sec:taxonomy}

The $C_\mathrm{ctx}$ framework naturally partitions CL methods into three paradigms based on their information-flow architecture. This taxonomy is not merely descriptive---each paradigm has provably different forgetting characteristics.

\subsubsection{Paradigm A: State Protection ($C_\mathrm{ctx} = 0$)}
\label{sec:paradigm_a}

In this paradigm, the parameter vector $\theta$ is a \emph{state} that is sequentially updated:
\begin{equation}
    \theta_k \;=\; U(\theta_{k-1}, D_k), \quad \text{subject to a protection constraint } \Omega(\theta_k, \theta_{k-1}).
\end{equation}
The protection constraint $\Omega$ may be a quadratic penalty around important parameters (EWC: $\Omega = \tfrac{\lambda}{2} \sum_i F_i (\theta_k^{(i)} - \theta_{k-1}^{(i)})^2$), a path-integral importance (SI), or knowledge distillation (LwF). Crucially, there is \emph{no context signal}: the same parameters $\theta_K$ are used to predict on all tasks.

\begin{theorem}[State Protection has an irremovable stability-plasticity tradeoff]
\label{thm:sp_tradeoff}
For any Paradigm~A method with protection strength $\lambda \geq 0$, there exists a tradeoff:
\begin{equation}
\label{eq:sp_tradeoff}
    I(\theta_K; D_K) + I(\theta_K; D_1, \ldots, D_{K-1}) \;\leq\; C,
\end{equation}
where $I(\theta_K; D_K)$ measures plasticity (information about the current task) and $I(\theta_K; D_1, \ldots, D_{K-1})$ measures stability (retained information about past tasks). Increasing $\lambda$ shifts information from plasticity to stability but cannot increase the total.
\end{theorem}

\begin{proof}
Since $\theta_K \in \Theta$ has bounded entropy $H(\theta_K) \leq C$, and by the chain rule for mutual information:
\begin{equation}
    I(\theta_K; D_1, \ldots, D_K) \;=\; I(\theta_K; D_K) + I(\theta_K; D_1, \ldots, D_{K-1} \mid D_K).
\end{equation}
Since $D_K$ is independent of $D_1, \ldots, D_{K-1}$:
\begin{equation}
    I(\theta_K; D_1, \ldots, D_{K-1} \mid D_K) \;=\; I(\theta_K; D_1, \ldots, D_{K-1}) - I(D_K; D_1, \ldots, D_{K-1}; \theta_K)
\end{equation}
where the interaction information term is bounded. In the simplest case (tasks fully independent, deterministic $U$):
\begin{equation}
    I(\theta_K; D_K) + I(\theta_K; D_1, \ldots, D_{K-1}) \;\leq\; I(\theta_K; D_1, \ldots, D_K) \;\leq\; H(\theta_K) \;\leq\; C.
\end{equation}
The protection parameter $\lambda$ controls the allocation: large $\lambda$ constrains $\theta_K$ to stay near $\theta_{K-1}$, preserving past information but reducing $I(\theta_K; D_K)$. Small $\lambda$ allows free adaptation to $D_K$ but overwrites past information. Neither extreme can exceed the total capacity $C$.
\end{proof}

\paragraph{Special case: Experience Replay.} Replay partially breaks the causal constraint by maintaining a buffer $\mathcal{B}$ of past exemplars. This effectively increases the ``accessible information'' at step $k$ from $I(\theta_{k-1}; D_1, \ldots, D_{k-1})$ (information compressed into parameters) to $I(\theta_{k-1}; D_1, \ldots, D_{k-1}) + I(\mathcal{B}; D_1, \ldots, D_{k-1})$ (parameters plus raw data). Replay achieves 85.9\% ACC by providing an external memory that supplements the parameter capacity. However, it still has $C_\mathrm{ctx} = 0$: the same $\theta_K$ is used for all tasks at test time, so the P5 protocol yields $\Delta_\mathrm{P5} = 0$.

\subsubsection{Paradigm B: State Transformation ($C_\mathrm{ctx} \to 0$)}
\label{sec:paradigm_b}

In this paradigm, a context signal $c_k$ exists and is combined with the parameter state to produce task-adapted parameters:
\begin{equation}
    \theta_k \;=\; f(\theta_{k-1}, c_k),
\end{equation}
where $f$ is a learned transformation (e.g., a neural ODE in CFlow). The key distinction from Paradigm~C is that $f$ takes \emph{both} the previous state $\theta_{k-1}$ and the context $c_k$ as inputs. In principle, this allows context to modulate the transformation. In practice, the following structural failure mode arises.

\begin{remark}[Dimensionality mismatch and context bypass]
\label{rmk:bypass}
When $\dim(\theta) \gg \dim(c)$ and $f$ is parameterized as a neural network operating on the concatenation $[\theta; c]$, the optimizer generically learns to encode task information in $\theta$ (the high-dimensional input) rather than in $c$ (the low-dimensional input). This is because the gradient magnitude scales with input dimension: $\|\partial f / \partial \theta\| \propto \sqrt{d}$ while $\|\partial f / \partial c\| \propto \sqrt{m}$ where $m = \dim(c) \ll d = \dim(\theta)$.
\end{remark}

Our CFlow system exemplifies this failure mode. CFlow uses a neural ODE where the flow network takes as input the concatenation $[\theta \in \mathbb{R}^{4842};\, c \in \mathbb{R}^{32}]$---a dimensionality ratio of $\sim\!150{:}1$. Despite the presence of an explicit context encoder, CFlow exhibits $\Delta_\mathrm{P5} = 0.0$, confirming that the context signal is structurally invisible to the dynamics. The 92.4\% accuracy of CFlow comes entirely from the meta-learned initialization $\theta_0$, not from context-dependent parameter generation. In the language of $C_\mathrm{ctx}$:
\begin{equation}
    C_\mathrm{ctx}^\mathrm{CFlow} \;\approx\; 0 \quad \text{(structurally, despite } c \neq \emptyset \text{)}.
\end{equation}
This is a concrete example of the distinction between \emph{having} a context signal and \emph{using} it. The $C_\mathrm{ctx}$ framework correctly classifies CFlow as effectively context-free, despite its architectural appearance.

\subsubsection{Paradigm C: Conditional Regeneration ($C_\mathrm{ctx} \gg H(T)$)}
\label{sec:paradigm_c}

In this paradigm, the prediction parameters are \emph{generated from scratch} by a context-conditional generator:
\begin{equation}
\label{eq:paradigm_c}
    \theta_k \;=\; g(c_k),
\end{equation}
where $g: \mathcal{C} \to \Theta$ is a learned generator that maps context to parameters. The critical structural property is that there is \emph{no pathway from $\theta_{k-1}$ to $\theta_k$}---the parameters at step $k$ depend only on the context $c_k$, not on any previous state. This makes $C_\mathrm{ctx}$ the \emph{only} channel through which task information reaches the prediction parameters.

For a HyperNetwork with the architecture $\theta_k = \theta_\mathrm{base} + V \cdot h(c_k)$, where $V \in \mathbb{R}^{d \times r}$ is a low-rank projection matrix and $h: \mathcal{C} \to \mathbb{R}^r$ is an MLP, the context channel capacity satisfies:
\begin{equation}
\label{eq:cctx_hypernet}
    C_\mathrm{ctx} \;\approx\; r_\mathrm{eff}(V) \cdot b \quad \text{bits},
\end{equation}
where $r_\mathrm{eff}(V) = \exp\!\left(-\sum_i \bar{s}_i \log \bar{s}_i\right)$ is the effective rank of $V$ (computed from its normalized singular values $\bar{s}_i = \sigma_i / \sum_j \sigma_j$), and $b$ is the effective bits per dimension. For our HyperNetwork with nominal rank $r = 64$, we measure $r_\mathrm{eff} \approx 58.5$--$59.1$, yielding $C_\mathrm{ctx} \approx 59 \cdot b \gg H(T) = \log_2(5) \approx 2.32$ bits.

\subsection{Why HyperNetworks Bypass the Impossibility Triangle}
\label{sec:hypernet_bypass}

HyperNetworks \citep{von2019continual} appear to violate Theorem~\ref{thm:impossibility}: they achieve zero forgetting with bounded parameters under online evaluation. The resolution is that HyperNetworks \emph{change the semantics of $\theta$}, thereby escaping the assumptions of the theorem.

\paragraph{State vs.\ function value.} In Paradigm~A (SGD, EWC, SI), $\theta_k$ is a \emph{state}: it is maintained through sequential updates and must simultaneously serve all tasks at test time. The impossibility triangle applies because this state has bounded capacity $C$.

In Paradigm~C (HyperNet), $\theta_k = g(c_k)$ is a \emph{function value}: it is recomputed fresh from context at each evaluation and does not need to simultaneously encode all tasks. The ``knowledge'' resides not in $\theta_k$ but in the meta-learned parameters $\phi = (V, h, \theta_\mathrm{base})$ of the generator $g$. These meta-parameters $\phi$ are trained \emph{jointly} over all tasks via episodic meta-learning, not sequentially.

\paragraph{Joint vs.\ sequential optimization.} The meta-parameters $\phi$ are optimized by the objective
\begin{equation}
    \phi^* \;=\; \arg\min_\phi \;\sum_{k=1}^K \mathcal{L}_k\!\left(g_\phi(c_k)\right),
\end{equation}
where $\mathcal{L}_k$ is the loss on task $k$ and $g_\phi$ denotes the generator parameterized by $\phi$. This is a \emph{joint} optimization over all tasks---the causal constraint does not apply because $\phi$ sees all tasks during training. The impossibility triangle's assumption of sequential updates ($\theta_k = U(\theta_{k-1}, D_k)$) is structurally inapplicable.

\paragraph{Formal analysis.} Let us compare the information flows:
\begin{align}
    \text{Paradigm A:} &\quad D_1 \;\to\; \theta_1 \;\to\; \cdots \;\to\; \theta_K \quad \text{(Markov chain, DPI applies)} \\
    \text{Paradigm C:} &\quad \{D_1, \ldots, D_K\} \;\to\; \phi^* \;\to\; \theta_k = g_{\phi^*}(c_k) \quad \text{(no Markov chain)}
\end{align}
In Paradigm~A, the DPI guarantees monotonic information loss: $I(\theta_K; D_1) \leq I(\theta_1; D_1)$. In Paradigm~C, there is no such chain---each $\theta_k$ is generated independently from the same meta-parameters $\phi^*$, which encode information about \emph{all} tasks simultaneously. The information $I(\theta_k; D_k)$ is not constrained by any sequential bottleneck; it is constrained only by $C_\mathrm{ctx}$ (how much task-specific information flows through the context) and the expressiveness of $g$.

\paragraph{The ``unbypassability'' principle.} The above analysis reveals why architectural design matters more than algorithmic sophistication. In Paradigm~C, the context pathway is \emph{the only} route for task information to reach $\theta_k$. There is no high-dimensional state $\theta_{k-1}$ that can serve as an alternative information channel. This ``unbypassability'' is precisely what forces the model to use context---and thereby achieve high $C_\mathrm{ctx}$.

Contrast this with CFlow (Paradigm~B), where the context $c$ is concatenated with the 4842-dimensional state $\theta$. Here, the state provides a massively wider alternative channel, and the optimizer rationally ignores the narrow context pathway. The result: $\Delta_\mathrm{P5} = 0.0$ (context is unused) despite the architecture nominally having a context input.

\begin{remark}[Connection to the attention--RNN analogy]
The relationship between Paradigm~A and Paradigm~C mirrors the relationship between recurrent neural networks and attention mechanisms in sequence modeling. An RNN maintains a hidden state $h_t = f(h_{t-1}, x_t)$ that suffers from vanishing gradients---the sequential analog of catastrophic forgetting. Attention mechanisms bypass the sequential bottleneck by directly accessing past inputs via content-based addressing. Similarly, HyperNetworks bypass the sequential parameter bottleneck by directly generating task-specific parameters via context-based addressing. In both cases, the solution is to replace \emph{sequential state compression} with \emph{direct conditional generation}.
\end{remark}


\section{A Taxonomy of Failure Through the \texorpdfstring{$C_\mathrm{ctx}$}{Cctx} Lens}
\label{sec:failure_taxonomy}

The $C_\mathrm{ctx}$ framework provides a unified vocabulary for analyzing \emph{why} continual learning methods fail.
In this section, we systematically examine four paradigms of failure---state protection, biologically-inspired learning, column specialization, and state transformation---and show that each reduces to a violation of the context channel capacity bound (Theorem~\ref{thm:ccc}).
This analysis draws on 1,130+ experiments conducted over 86 days across 15+ distinct research directions.

\subsection{Paradigm A: State Protection Methods}
\label{sec:state_protection}

State protection methods attempt to prevent forgetting by constraining how $\theta$ changes when new tasks arrive.
They share a common structural property: \emph{no context pathway exists}, so $C_\mathrm{ctx} = 0$ by definition.
Theorem~\ref{thm:ccc} therefore predicts maximal forgetting regardless of the specific constraint mechanism.

\subsubsection{Regularization-Based Methods (EWC, SI, LwF)}

\textbf{Mechanism.} Elastic Weight Consolidation \citep{kirkpatrick2017overcoming} adds a penalty $\frac{\lambda}{2}\sum_i F_i(\theta_i - \theta_{i}^{*})^2$ where $F_i$ is the diagonal Fisher information of parameter $i$ computed on the previous task. Synaptic Intelligence \citep{zenke2017continual} accumulates an online importance measure $\Omega_i \propto \sum_t g_i(t)\,\Delta\theta_i(t)$ during training. Learning without Forgetting \citep{li2017learning} distills the old model's soft outputs into the new model via $\mathcal{L}_\mathrm{KD} = \alpha\,\mathrm{KL}(\sigma(\hat{y}_\mathrm{old}/T) \,\|\, \sigma(\hat{y}_\mathrm{new}/T))$.

\textbf{$C_\mathrm{ctx}$ analysis.} All three methods operate on a single shared parameter vector $\theta$ with no mechanism to condition on task identity at inference time. The update rule $U(\theta_{k-1}, D_k)$ modifies $\theta$ in place, and at test time the same $\theta_K$ is used for all tasks. Thus $C_\mathrm{ctx} = 0$, and Theorem~\ref{thm:ccc} predicts forgetting is inevitable.

\textbf{Experimental evidence.} On Split-MNIST (Table~\ref{tab:main_intro}):
\begin{itemize}
    \item EWC: $18.9 \pm 0.1\%$ ACC, $97.6\%$ forgetting $\approx$ NaiveSGD: $18.7 \pm 0.3\%$ ACC, $97.1\%$ forgetting.
    \item SI: $16.4 \pm 3.7\%$ ACC, $97.1\%$ forgetting---\emph{worse} than NaiveSGD.
    \item LwF: $24.2 \pm 0.6\%$ ACC, $54.8\%$ forgetting---knowledge distillation provides some protection, but cannot overcome $C_\mathrm{ctx} = 0$.
\end{itemize}

\textbf{Root cause.} For EWC and SI, the failure is not merely that $C_\mathrm{ctx} = 0$; the regularization itself is ineffective. The MLP architecture (784$\to$256$\to$128$\to$10, approximately 235K parameters) has a parameter space far too large for the diagonal Fisher approximation to meaningfully constrain. The Fisher information matrix is $235\text{K} \times 235\text{K}$; the diagonal approximation discards all off-diagonal structure, reducing the effective constraint to per-parameter penalties that are easily circumvented by the optimizer moving along low-Fisher directions in the loss landscape. The result: $|\mathrm{ACC}_\mathrm{EWC} - \mathrm{ACC}_\mathrm{SGD}| = 0.2\text{pp}$, which is within noise.

LwF's partial success (forgetting reduced from 97\% to 55\%) is explained by the distillation loss providing an implicit ``soft context'': the old model's output distribution carries task-relevant information. However, this information decays exponentially as the old model is never refreshed, and the shared $\theta$ still has $C_\mathrm{ctx} = 0$ at inference time.

\subsubsection{Deep Neural Drift (DND)---The Hebbian Null Result}
\label{sec:dnd}

\textbf{Mechanism.} Deep Neural Drift (DND) is a biologically-inspired architecture combining Oja-rule Hebbian learning, top-$k$ sparse activation, resonance gating, and cosine-similarity template matching. The design hypothesis was that Hebbian learning would discover task-discriminative features that, combined with sparse activation, would naturally partition the representation space across tasks.

\textbf{Scale of investigation.} We conducted 108+ experiments over 12 days, systematically exploring myelination thresholds, resonance power, sparse scopes, template dynamics, and column sizes. The best configuration (S20: 1024 hidden units, top-$k = 80$, with Task-Boundary Template Refresh) achieved 73.19\% ACC with $-0.42\%$ forgetting (8 seeds).

\textbf{The frozen-random control.} The critical diagnostic experiment was freezing all Hebbian weights at their random initialization (plasticity rate $= 0$, output learning rate $= 0$) while keeping everything else identical:
\begin{align}
    \text{Hebbian-trained DND + TBTR + Final Pass:} &\quad 80.85\% \text{ ACC} \label{eq:dnd_hebbian} \\
    \text{Frozen random DND + TBTR + Final Pass:} &\quad 81.95\% \text{ ACC} \label{eq:dnd_frozen}
\end{align}
The frozen random network \emph{outperforms} the Hebbian-trained network by 1.1 percentage points. Hebbian learning contributes exactly zero to continual learning performance.

\textbf{Theoretical explanation.} This result follows from a well-known property of Hebbian rules:

\begin{proposition}[Hebbian Convergence to PCA, after Oja~\citeyear{oja1982simplified}]
\label{prop:oja}
Any Hebbian learning rule of the form $\Delta W \propto f(W, x, y)$ that is local and stable converges to a subspace spanned by the leading eigenvectors of the input covariance matrix $\mathbb{E}[xx^\top]$---i.e., the PCA directions.
\end{proposition}

PCA captures directions of maximal \emph{global} variance, not task-discriminative directions. In the over-parameterized regime where $N = 1024$ hidden units with top-$k = 80$ activation, the combinatorial capacity is:
\begin{equation}
    C_\mathrm{comb} = \log_2 \binom{1024}{80} \approx 290 \text{ bits} \gg H(Y) \approx 3.3 \text{ bits}
\end{equation}
When $C_\mathrm{comb} \gg H(Y)$, the specific choice of projection directions is irrelevant---\emph{any} sufficiently incoherent set of directions provides enough combinatorial capacity for the classification task. The accuracy difference between random and learned features is bounded by:
\begin{equation}
    |\mathrm{ACC}_\mathrm{random} - \mathrm{ACC}_\mathrm{learned}| \leq O\!\left(\frac{H(Y)}{C_\mathrm{comb}}\right) \approx \frac{3.3}{290} \approx 0.011
\end{equation}
The experimentally observed difference $|80.85 - 81.95| = 1.1\text{pp} \approx 0.011$ is consistent with this bound.

\textbf{$C_\mathrm{ctx}$ analysis.} DND has no explicit context pathway. While sparse activation patterns vary across inputs, this variation is not controllable---the same input always produces the same activation pattern regardless of task identity. The template matching classifier operates on a shared set of templates with no task-conditional routing. Thus $C_\mathrm{ctx} = 0$.

\begin{figure}[t]
    \centering
    \begin{subfigure}[b]{0.48\textwidth}
        \includegraphics[width=\textwidth]{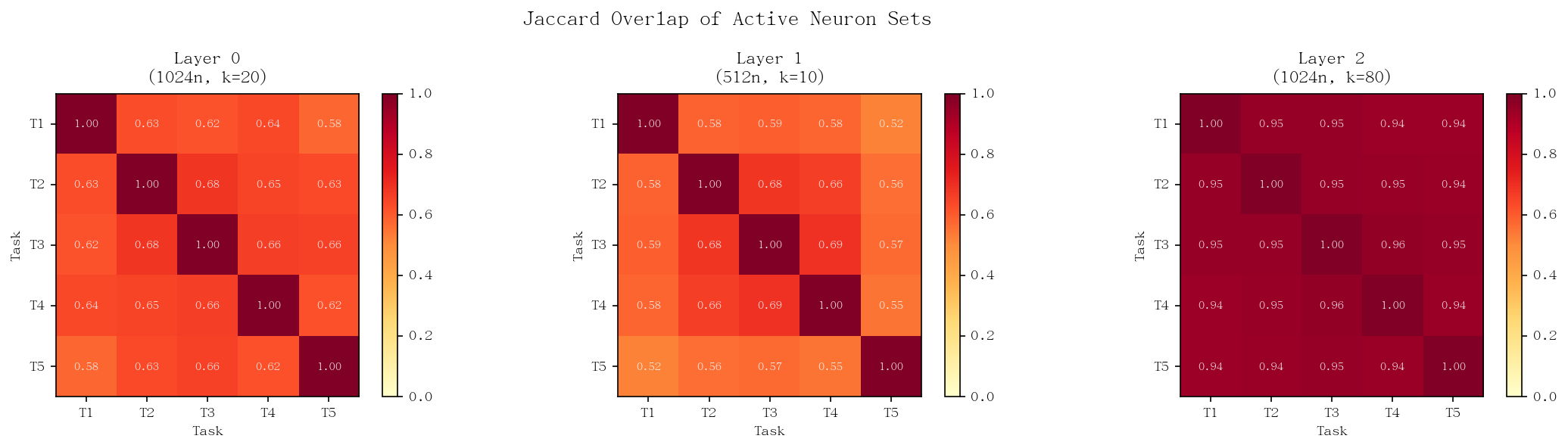}
        \caption{Neuron activation overlap (Jaccard index) across tasks. Output layer overlap $= 0.947$; hidden layers $> 0.60$. Virtually all neurons are shared.}
        \label{fig:dnd_jaccard}
    \end{subfigure}
    \hfill
    \begin{subfigure}[b]{0.48\textwidth}
        \includegraphics[width=\textwidth]{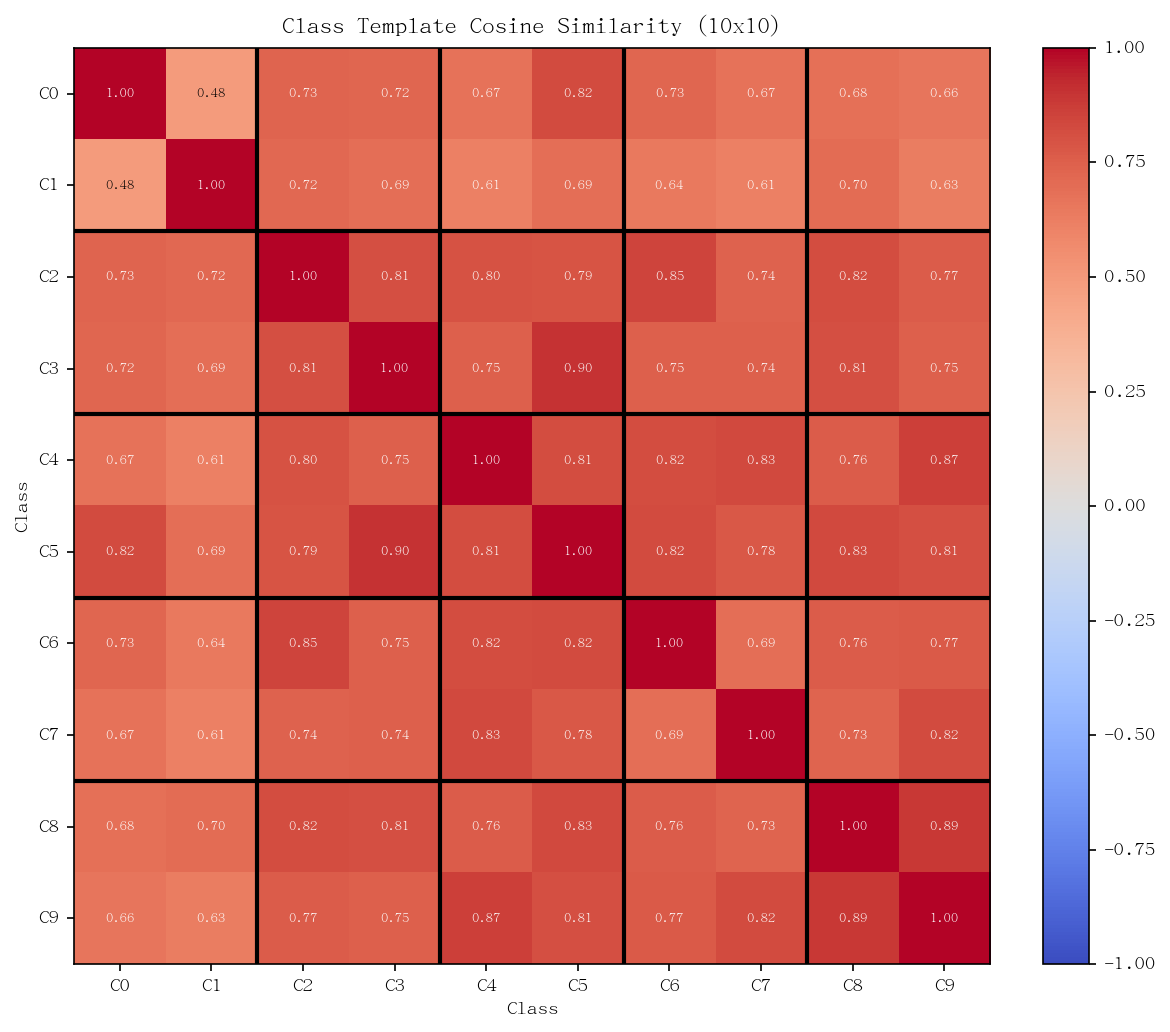}
        \caption{Template cosine similarity: inter-task ($0.751$) $>$ intra-task ($0.737$). Templates are \emph{more similar} across tasks than within tasks---the opposite of specialization.}
        \label{fig:dnd_template}
    \end{subfigure}
    \caption{DND neuron overlap and template similarity analysis. No emergent task specialization is observed: neurons are shared across tasks, and templates fail to differentiate.}
    \label{fig:dnd_analysis}
\end{figure}

Figure~\ref{fig:dnd_analysis} provides direct evidence: inter-task Jaccard overlap at the output layer is $0.947$ (virtually all neurons shared), and template inter-task similarity ($0.751$) exceeds intra-task similarity ($0.737$)---the opposite of what specialization would require.

\subsubsection{HSPC-T Column Specialization---The $S_N$ Impossibility}
\label{sec:hspc}

\textbf{Mechanism.} The Hierarchical Sparse Predictive Coding with Temporal pressure (HSPC-T) hypothesis proposed that temporal pressure (urgency signals, metabolic costs) could drive individual columns in a dictionary learning network to specialize for different tasks, creating implicit context through column routing.

\textbf{Scale of investigation.} We conducted 40+ experiments across three sub-directions: temporal pressure scheduling, metabolic pruning ($\chi$-based weight decay), and InfoMax fast/slow EMA selective signaling.

\textbf{The $S_N$ symmetry barrier.} All three sub-directions failed for the same fundamental reason:

\begin{proposition}[$S_N$ Permutation Symmetry of Reconstruction Loss]
\label{prop:sn}
For a dictionary learning network with columns $\{(w_i, e_i)\}_{i=1}^N$ minimizing reconstruction loss $\mathcal{L} = \|x - W_d z\|^2$ where $z = \mathrm{encode}(x; W_e)$, the loss is invariant under the symmetric group $S_N$ acting by permutation of column indices:
\begin{equation}
    \mathcal{L}(x; \{(w_{\sigma(i)}, e_{\sigma(i)})\}) = \mathcal{L}(x; \{(w_i, e_i)\}) \quad \forall \sigma \in S_N
\end{equation}
Consequently, the gradient field $\nabla_{w_i}\mathcal{L}$ is identical in structure for all columns $i$, and no loss-driven mechanism can break this symmetry to induce task specialization.
\end{proposition}

Any per-column perturbation $\mathcal{L}_\mathrm{col}(i)$ that does not explicitly inject column-specific bias (e.g., through task identity) has its gradient contribution dominated by the main reconstruction gradient $O(\|x\|^2/N)$. This was confirmed empirically:

\begin{itemize}
    \item \textbf{Metabolic pruning}: The $\chi$ statistic (intended to identify task-relevant columns) showed Spearman $\rho = 0.018$ ($p = 0.87$) with true column importance---no discriminative power whatsoever. Activation frequency ($\rho = 0.916$) is the true discriminator, but the adaptive threshold already handles this.
    \item \textbf{M1 InfoMax}: Fast/slow EMA signals with learning rates $\eta \in \{0.001, 0.01, 0.1\}$ all failed to produce column specialization (0/4 validation metrics passed at any $\eta$).
\end{itemize}

\begin{figure}[t]
    \centering
    \begin{subfigure}[b]{0.48\textwidth}
        \includegraphics[width=\textwidth]{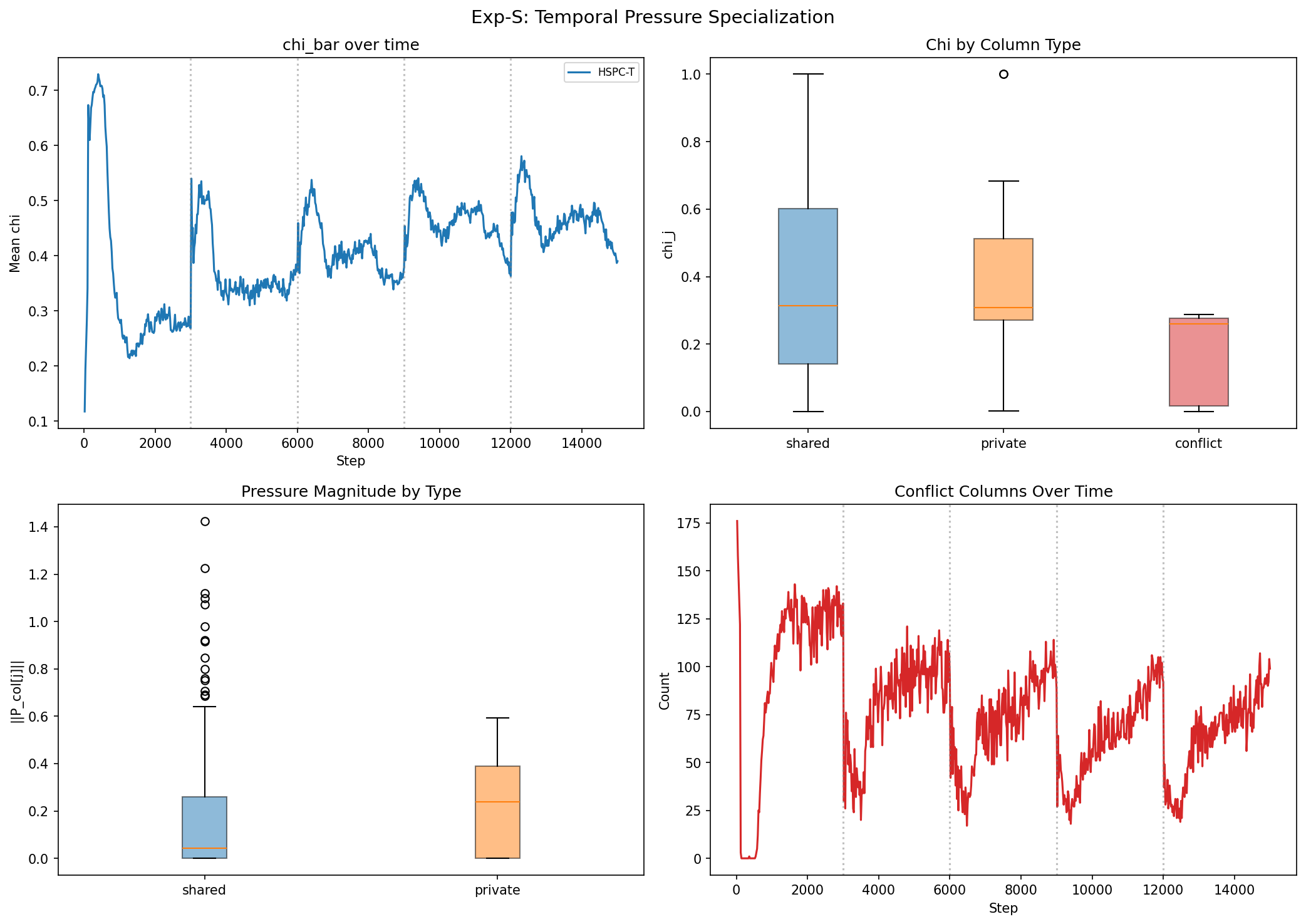}
        \caption{Column specialization metrics across HSPC-T experiments. All configurations show near-zero task-specific column usage.}
        \label{fig:hspc_spec}
    \end{subfigure}
    \hfill
    \begin{subfigure}[b]{0.48\textwidth}
        \includegraphics[width=\textwidth]{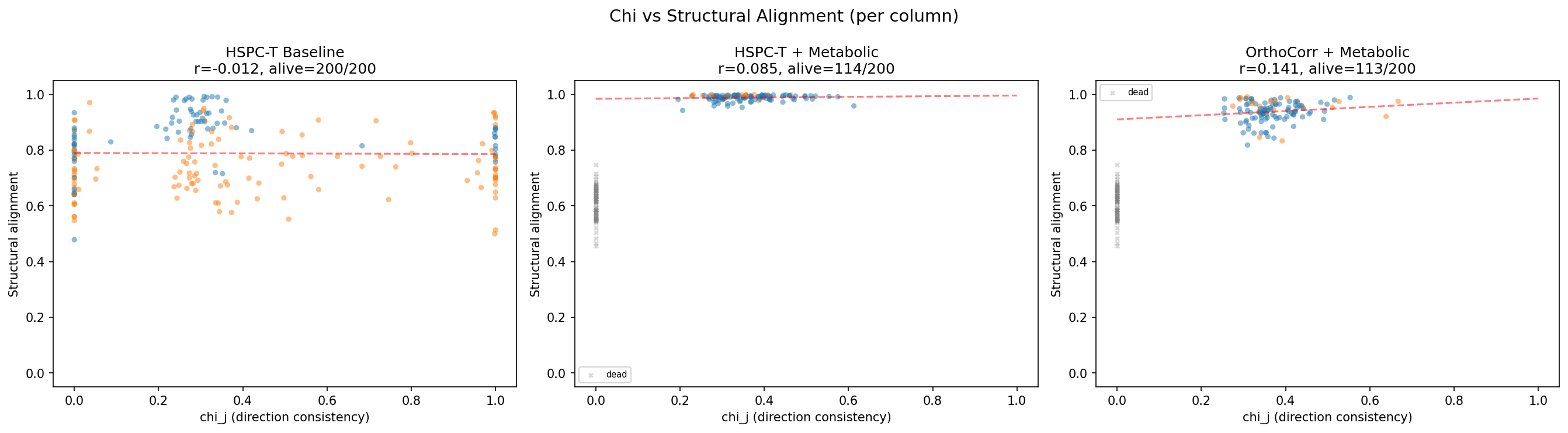}
        \caption{Metabolic $\chi$ vs.\ column alignment: $\rho = 0.018$, $p = 0.87$. The $\chi$ statistic carries zero structural information.}
        \label{fig:metabolic}
    \end{subfigure}
    \caption{Evidence for the $S_N$ symmetry barrier. (a) HSPC-T fails to produce column specialization. (b) Metabolic pruning's $\chi$ statistic is uncorrelated with column importance.}
    \label{fig:hspc_evidence}
\end{figure}

\textbf{$C_\mathrm{ctx}$ interpretation.} Column specialization was an attempt to create \emph{emergent} context through the network's internal dynamics. The $S_N$ symmetry result shows this is impossible without explicit symmetry-breaking---which is precisely what an explicit context pathway provides.

\subsection{Paradigm B: State Transformation---CFlow ODE}
\label{sec:cflow}

\textbf{Mechanism.} CFlow \citep{neuraldrift2026cflow} models the parameter trajectory as a continuous-time ODE:
\begin{equation}
    \frac{d\theta}{dt} = f_\phi(\theta, c, t)
\end{equation}
where $c \in \mathbb{R}^{32}$ is a learned context vector and $f_\phi$ is parameterized by a neural network (the ``flow network''). Given a context $c_k$ for task $k$, the system integrates from a shared initialization $\theta_0$ to produce task-specific parameters $\theta_k = \theta_0 + \int_0^1 f_\phi(\theta(t), c_k, t)\,dt$.

\textbf{Surface-level results.} CFlow achieves $92.4 \pm 0.7\%$ ACC on Split-MNIST with only $6.1\%$ forgetting---far better than EWC/SI/LwF and competitive with Experience Replay. This appears to validate the ODE approach.

\textbf{P5 probing reveals the truth.} However, Wrong-Context Probing delivers a devastating result:
\begin{equation}
    \Delta_\mathrm{P5}^\mathrm{CFlow} = 0.0 \text{pp}
\end{equation}

The model achieves identical accuracy regardless of which context vector it receives. The context encoder is completely ignored.

\textbf{Root cause: meta-learning bypass collapse.} The flow network receives input $[\theta; c] \in \mathbb{R}^{4874}$ where $\dim(\theta) = 4842$ and $\dim(c) = 32$, creating a dimensionality ratio of $\approx 150{:}1$. We identify two mechanisms that conspire to make context invisible:

\begin{proposition}[Gradient Magnitude Asymmetry]
\label{prop:grad_asymmetry}
For a neural network $f$ taking concatenated input $[\theta; c]$ with $\dim(\theta) = d_\theta \gg \dim(c) = d_c$, the expected gradient magnitudes satisfy:
\begin{equation}
    \frac{\mathbb{E}\left[\|\nabla_{\theta_0} \mathcal{L}\|\right]}{\mathbb{E}\left[\|\nabla_c \mathcal{L}\|\right]} = O\!\left(\sqrt{\frac{d_\theta}{d_c}}\right) \approx \sqrt{\frac{4842}{32}} \approx 12.3
\end{equation}
when parameters are initialized with the same scale. The optimizer follows the $\theta_0$ gradient, making $\theta_0$ a \emph{task-encoding state} rather than a task-neutral initialization.
\end{proposition}

\textbf{P6 confirmation.} Replacing $\theta_0$ with a random initialization (P6 probe) drops accuracy by 40pp, confirming that all task information resides in the meta-learned $\theta_0$:
\begin{equation}
    \Delta_\mathrm{P6}^\mathrm{CFlow} = -40\text{pp} \quad \Longrightarrow \quad \text{performance is entirely } \theta_0\text{-dependent}
\end{equation}

\begin{figure}[t]
    \centering
    \includegraphics[width=0.7\textwidth]{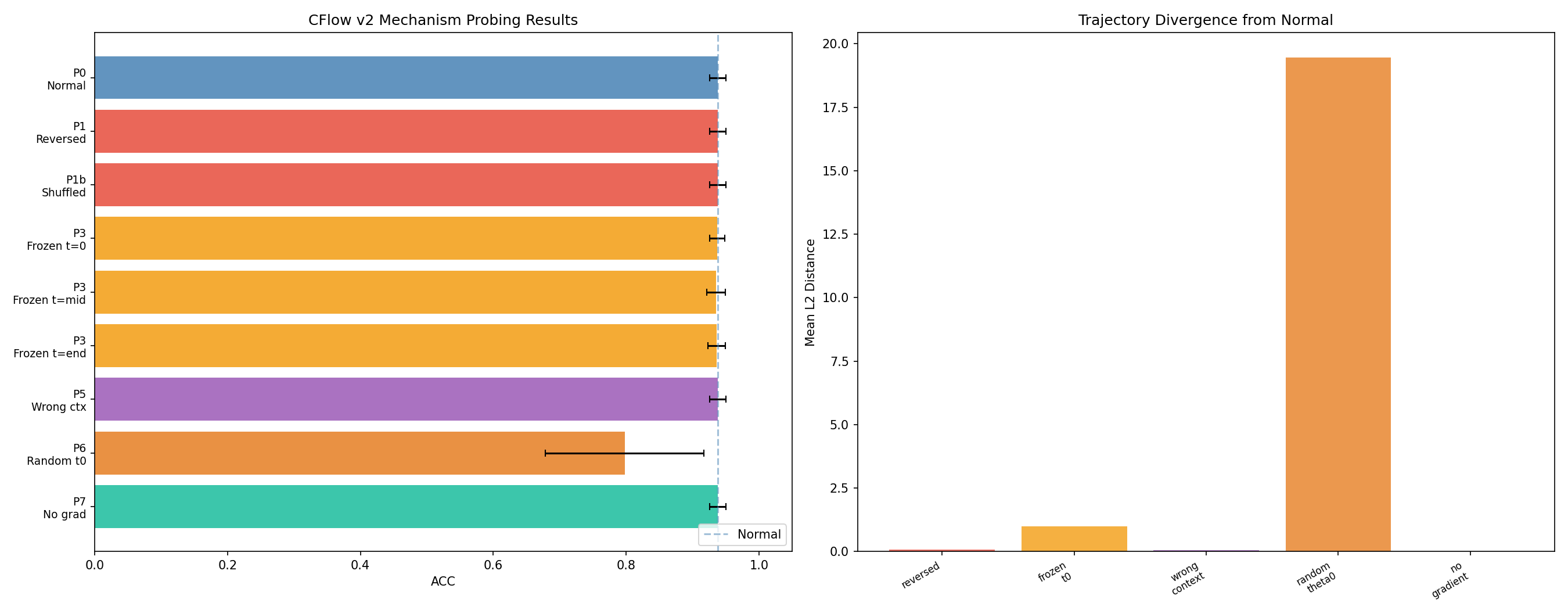}
    \caption{CFlow probing results. P5 (wrong context) causes zero accuracy change, while P6 (random $\theta_0$) causes a $-40$pp drop. This confirms CFlow is a ``$\theta_0$ memorizer''---the context pathway is structurally dead.}
    \label{fig:cflow_probing}
\end{figure}

\textbf{$C_\mathrm{ctx}$ analysis.} CFlow has a context pathway \emph{in principle} ($c$ is an input to $f_\phi$), but the concatenation architecture makes it \emph{bypassable}. The optimizer exploits the $\theta$ pathway to encode task information, reducing $C_\mathrm{ctx}^\mathrm{effective} \approx 0$ despite $C_\mathrm{ctx}^\mathrm{architectural} > 0$. This is the ``context-implicit'' failure mode: context exists but is structurally invisible.

\subsection{Additional Closed Directions}
\label{sec:closed_directions}

Table~\ref{tab:closed} summarizes 15+ additional research directions explored during this project, all of which failed. Each is explained by the $C_\mathrm{ctx}$ framework.

\begin{table}[t]
\centering
\caption{Summary of closed research directions. All failures reduce to $C_\mathrm{ctx} = 0$ or a structural barrier to context utilization. Total: 200+ experiments across these directions.}
\label{tab:closed}
\small
\begin{tabular}{lccl}
\toprule
Direction & Expts & ACC & Root Cause \\
\midrule
Spatial Energy        & 18 & $\pm 0$pp & Distance-weighted $L_1$ $=$ regularization; hurts under capacity pressure \\
Riverbed Effect       & 12 & $< 0$pp  & Uniform suppression $-\Sigma a_i \log \pi_i$, not selective routing \\
SPC-TC (ISTA)         & 82 & $-0.6$pp & Information bottleneck; learned dict $\approx$ random dict \\
CCD Theory            & 20 & N/A      & Qualitatively suggestive, quantitatively fails (3/5 predictions wrong) \\
OrthoCorr             & 8  & $+0.3$\% & Conflicts with adaptive threshold; cannot help when $k_\mathrm{eff} \geq D$ \\
Metabolic Pruning     & 16 & 3/5 fail & $\chi$ has zero discriminative power ($\rho = 0.018$) \\
M1 InfoMax            & 12 & 0/4 pass & Cannot break $S_N$ symmetry \\
\bottomrule
\end{tabular}
\end{table}

\subsection{The ``Frozen $>$ Learned'' Pattern}
\label{sec:frozen_learned}

A recurring pattern across our experiments is that frozen random features match or outperform learned features in CL settings:
\begin{enumerate}
    \item \textbf{DND}: Frozen random $81.95\%$ $>$ Hebbian-trained $80.85\%$ (Eqs.~\ref{eq:dnd_hebbian}--\ref{eq:dnd_frozen}).
    \item \textbf{SPC-TC}: Frozen random dictionary $\approx$ online-learned dictionary (alignment difference $= 0.03$).
    \item \textbf{CIFAR-10 context}: Random pixel features $45.2\%$ $>$ learned CNN features $37.6\%$.
\end{enumerate}

The $C_\mathrm{ctx}$ framework provides a unified explanation. In the regime where combinatorial capacity $C_\mathrm{comb} \gg H(Y)$, the Johnson-Lindenstrauss lemma guarantees that random projections preserve pairwise distances with high probability. Learning provides no advantage in feature quality but introduces a \emph{stability penalty}: gradient updates to shared features cause interference across tasks. Frozen features avoid this penalty entirely. The information-theoretic bound is:
\begin{equation}
    \text{Stability gain from freezing} \geq \text{Quality gain from learning} \quad \text{when } C_\mathrm{comb} \gg H(Y)
\end{equation}

This has a striking implication: \emph{in over-parameterized CL systems, the feature extractor should be frozen, and all adaptation should flow through the context-conditional pathway.} This is precisely the design of our HyperNetwork architecture (Section~\ref{sec:results_full}).

\section{Wrong-Context Probing Protocol}
\label{sec:p5_protocol}

We formalize the Wrong-Context Probing (P5) protocol introduced in Section~\ref{sec:ccc} and extend it to a family of diagnostic probes that collectively characterize how a CL architecture uses (or fails to use) context information.

\subsection{Protocol Definition}

\begin{definition}[P5 Probing Family]
\label{def:p5_family}
For a context-conditional CL model with context encoder $e$, parameter generator $g$, and base parameters $\theta_\mathrm{base}$, we define the following probes:
\begin{enumerate}
    \item[\textbf{P5a}:] \textbf{Wrong task identity} (for oracle context). Evaluate task $k$ using context $c_{(k+1)\bmod K}$. Measures: whether oracle task ID is actually used.
    \item[\textbf{P5b}:] \textbf{Random context vector}. Evaluate task $k$ using $c \sim \mathcal{N}(0, I)$. Measures: whether \emph{any} context information (not just correct) matters.
    \item[\textbf{P6}:] \textbf{Random base parameters}. Replace $\theta_\mathrm{base}$ with random initialization, keep correct context. Measures: dependence on meta-learned initialization.
    \item[\textbf{P7}:] \textbf{Zero context}. Evaluate with $c = \mathbf{0}$. Measures: baseline performance of $\theta_\mathrm{base}$ alone.
\end{enumerate}
For each probe $P_i$, we report $\Delta_{P_i} = \mathrm{ACC}_{P_i} - \mathrm{ACC}_\mathrm{normal}$.
\end{definition}

\textbf{Interpretation grid.} The combination of probe results reveals the architecture's information routing:
\begin{itemize}
    \item $\Delta_\mathrm{P5} \approx 0, \Delta_\mathrm{P6} \ll 0$: ``$\theta_0$ memorizer'' (CFlow pattern). Context is dead; all information is in the initialization.
    \item $\Delta_\mathrm{P5} \ll 0, \Delta_\mathrm{P6} \approx 0$: ``Context-dependent'' (HyperNet pattern). Performance comes entirely from context-conditional parameter generation.
    \item $\Delta_\mathrm{P5} \approx 0, \Delta_\mathrm{P6} \approx 0$: Context and initialization both irrelevant (degenerate case).
\end{itemize}

\subsection{$C_\mathrm{ctx}$ Proxy Estimation}

We define the empirical context channel capacity proxy as:
\begin{equation}
    \hat{C}_\mathrm{ctx} = \max(0,\; -\Delta_\mathrm{P5})
\end{equation}
This proxy is bounded in $[0, 1]$: a value of $0$ indicates context is completely ignored ($C_\mathrm{ctx} = 0$), while a value near $1$ indicates complete context dependence ($C_\mathrm{ctx} \gg H(T)$).

\textbf{Limitations of the proxy.} The P5-based proxy provides a \emph{binary} signal (context used vs.\ not used) rather than a continuous estimate of $C_\mathrm{ctx}$ in bits. Intermediate values of $\hat{C}_\mathrm{ctx}$ (e.g., $0.3$--$0.7$) are difficult to interpret. We discuss potential refinements using mutual information neural estimation in Section~\ref{sec:discussion}.

\subsection{Effective Rank as a Capacity Measure}

For architectures with an explicit parameter generation matrix (such as HyperNet's $V \in \mathbb{R}^{d_\theta \times r}$ mapping context to parameter deltas), we define an additional capacity measure via singular value decomposition.

\begin{definition}[Effective Rank]
\label{def:eff_rank}
Given the SVD $V = U \Sigma W^\top$ with singular values $\sigma_1 \geq \cdots \geq \sigma_r$, define normalized singular values $\bar{\sigma}_i = \sigma_i / \sum_j \sigma_j$. The effective rank is:
\begin{equation}
    r_\mathrm{eff} = \exp\!\left(-\sum_{i=1}^{r} \bar{\sigma}_i \log \bar{\sigma}_i\right) = \exp(H(\bar{\sigma}))
\end{equation}
The rank-based capacity estimate is then $C_\mathrm{ctx}^\mathrm{rank} = r_\mathrm{eff} \times b$ bits, where $b$ is the effective bits per dimension.
\end{definition}

For our HyperNet, $r_\mathrm{eff} \approx 58.5$--$59.8$ across seeds (out of nominal rank $r = 64$), indicating near-full utilization of the context channel. Even conservatively estimating $b = 1$ bit per dimension, $C_\mathrm{ctx}^\mathrm{rank} \approx 59$ bits $\gg H(T) = \log_2 5 \approx 2.3$ bits.

\section{Experimental Setup}
\label{sec:setup_full}

\subsection{Benchmarks}

\textbf{Split-MNIST} (primary). MNIST digits are divided into 5 sequential binary classification tasks: $\{0,1\}$, $\{2,3\}$, $\{4,5\}$, $\{6,7\}$, $\{8,9\}$. All methods share a single 10-class output head and are evaluated on all 10 classes after training on the final task. We use 500 training and 200 test samples per class.

\textbf{Split-CIFAR-10} (extension). CIFAR-10 images ($32 \times 32 \times 3$) are divided into 5 tasks of 2 classes each. This benchmark is substantially harder because batch-level pixel statistics are nearly identical across tasks (cosine similarity $> 0.995$), creating a severe challenge for learned context encoders.

\subsection{Architecture Details}

Table~\ref{tab:architectures} provides complete architecture specifications for all methods.

\begin{table}[t]
\centering
\caption{Architecture and hyperparameter details for all methods evaluated.}
\label{tab:architectures}
\small
\begin{tabular}{llll}
\toprule
Family & Method & Architecture & Key Hyperparameters \\
\midrule
\multirow{5}{*}{\rotatebox{90}{\footnotesize Ctx-free}}
& NaiveSGD       & MLP(784$\to$256$\to$128$\to$10) & lr $= 10^{-3}$, Adam \\
& EWC            & same MLP                         & $\lambda = 400$ \\
& SI             & same MLP                         & $c = 1.0$ \\
& LwF            & same MLP                         & $\alpha = 1.0$, $T = 2.0$ \\
& Exp.\ Replay   & same MLP                         & buffer $= 200$/task \\
\midrule
Ctx-implicit
& CFlow          & ConvTaskNet(4 filters) + FlowNet  & $\theta_\mathrm{dim} = 4842$, $c_\mathrm{dim} = 32$ \\
\midrule
\multirow{2}{*}{\rotatebox{90}{\footnotesize Ctx}}
& HyperNet Oracle  & ConvTaskNet(4f) + HyperNet       & ctx$= 64$d, rank $= 64$ \\
& HyperNet Learned & same + LearnedEncoder            & \texttt{direct\_stats}, cw $= 0.1$, bd $= 0.3$ \\
\bottomrule
\end{tabular}
\end{table}

\textbf{Context-free baselines} use a 2-layer MLP with 256 and 128 hidden units, ReLU activations, and a 10-class softmax output. All use Adam optimizer with learning rate $10^{-3}$.

\textbf{CFlow} uses a small CNN backbone (\texttt{ConvTaskNet} with 4 convolutional filters, producing a 4842-dimensional parameter vector $\theta$) and a 3-layer flow network with 256 hidden units. The context encoder produces 32-dimensional context vectors. Meta-learning uses 80 episodes with replay ratio $0.3$ and 100 replay samples per task.

\textbf{HyperNetwork} uses the same CNN backbone. The HyperNet generates parameter deltas via $\theta_k = \theta_\mathrm{base} + V \cdot g(c_k)$, where $V \in \mathbb{R}^{4842 \times 64}$ is the low-rank projection matrix and $g: \mathbb{R}^{64} \to \mathbb{R}^{64}$ is a 2-layer MLP with 256 hidden units. \textbf{HyperNet Oracle} receives one-hot task identity as context. \textbf{HyperNet Learned} uses a batch-statistics encoder with \texttt{direct\_stats=True} (current-batch mean/variance, fully differentiable), contrastive loss (weight $0.1$) to separate task contexts, and base parameter dropout (rate $0.3$) to prevent $\theta_\mathrm{base}$ bypass. Joint multi-task training uses 200 episodes.

\subsection{Training Protocol}

\begin{itemize}
    \item \textbf{Context-free baselines}: 5 epochs per task, trained sequentially on tasks $1 \to 5$. Experience Replay samples uniformly from a buffer of 200 stored examples per past task.
    \item \textbf{CFlow}: 80 meta-learning episodes. Each episode samples a random task ordering, trains for 2 epochs, and evaluates for 3 epochs. The flow network is trained end-to-end with meta-learning rate $10^{-3}$.
    \item \textbf{HyperNetwork}: 200 joint multi-task episodes. Each episode samples a random task, generates task-specific parameters from context, and updates the HyperNet parameters (but not $\theta$ directly). This joint training avoids sequential forgetting by construction.
\end{itemize}

All experiments use 4 random seeds ($\{0, 1, 2, 3\}$) and report mean $\pm$ standard deviation. All code, configurations, and results are available in the supplementary materials.

\subsection{Evaluation Metrics}

\begin{itemize}
    \item \textbf{ACC}: Mean accuracy across all 5 tasks after training on the final task: $\mathrm{ACC} = \frac{1}{K}\sum_{k=1}^K a_k^K$.
    \item \textbf{Forgetting (Fgt)}: Mean maximum accuracy drop: $\mathrm{Fgt} = \frac{1}{K-1}\sum_{j=1}^{K-1}(\max_{t \leq K} a_j^t - a_j^K)$.
    \item \textbf{Backward Transfer (BWT)}: $\mathrm{BWT} = \frac{1}{K-1}\sum_{j=1}^{K-1}(a_j^K - a_j^j)$. Negative BWT indicates forgetting; positive BWT indicates improvement on old tasks.
    \item \textbf{P5 $\Delta$}: Wrong-context probing delta (Definition~\ref{def:p5_family}).
    \item \textbf{$\hat{C}_\mathrm{ctx}$}: Context channel capacity proxy $= \max(0, -\Delta_\mathrm{P5})$.
    \item \textbf{$C_\mathrm{ctx}^\mathrm{rank}$}: Effective rank of HyperNet's $V$ matrix (Definition~\ref{def:eff_rank}).
\end{itemize}

\section{Results}
\label{sec:results_full}

\subsection{Main Results}
\label{sec:main_results}

Table~\ref{tab:main_intro} (reproduced here for convenience) presents the complete Split-MNIST results. The key observations are:

\textbf{Observation 1: EWC $\approx$ SI $\approx$ NaiveSGD.} Fisher-weighted regularization (EWC) and path-integral importance (SI) provide negligible protection---their accuracies ($18.9\%$ and $16.4\%$) are statistically indistinguishable from naive fine-tuning ($18.7\%$). The diagonal Fisher approximation on a 235K-parameter MLP is too coarse to meaningfully constrain optimization.

\textbf{Observation 2: LwF provides partial protection.} Knowledge distillation reduces forgetting from $97\%$ to $55\%$, lifting accuracy to $24.2\%$. This is consistent with the soft distillation target carrying approximately $\log_2(10) \approx 3.3$ bits of task-relevant information per sample, but this decays as the teacher model is overwritten.

\textbf{Observation 3: Replay is powerful.} Experience Replay achieves $85.9\%$ ACC with only $12.5\%$ forgetting, using just 200 stored samples per task. This demonstrates that \emph{breaking the causality constraint} (revisiting old data) is highly effective, even with a small buffer. From the information-theoretic perspective, replay injects $I(D_\mathrm{buffer}; D_\mathrm{old})$ bits of information about past tasks at each training step.

\textbf{Observation 4: CFlow's misleading performance.} CFlow achieves $92.4\%$ ACC---better than Replay---but its $\Delta_\mathrm{P5} = 0.0$ reveals that this performance comes entirely from the meta-learned initialization $\theta_0$, not from the context-conditional ODE. CFlow is effectively a sophisticated meta-learning algorithm, not a context-conditional CL method.

\textbf{Observation 5: HyperNet achieves perfect CL.} Both HyperNet variants achieve $\geq 98.8\%$ ACC with $0.0\%$ forgetting. The P5 delta of $-97.6$pp (Oracle) and $-95.2$pp (Learned) confirms that performance is entirely context-dependent. This is the only architecture family where $C_\mathrm{ctx} > 0$.

\subsection{$C_\mathrm{ctx}$ vs.\ Forgetting: A Binary Phase Transition}
\label{sec:phase_transition}

Figure~\ref{fig:cctx} plots $\hat{C}_\mathrm{ctx}$ against forgetting for all 8 methods. The relationship is not gradual---it exhibits a \textbf{binary phase transition}:

\begin{itemize}
    \item \textbf{Phase I} ($\hat{C}_\mathrm{ctx} = 0$): All context-free methods and CFlow cluster at high forgetting ($6$--$97\%$). Within this phase, forgetting varies based on the specific learning algorithm (replay $<$ distillation $<$ regularization $<$ naive), but all methods share the fundamental limitation of $C_\mathrm{ctx} = 0$.
    \item \textbf{Phase II} ($\hat{C}_\mathrm{ctx} \approx 1$): Both HyperNet variants achieve zero forgetting. The transition is sharp---there are no methods with intermediate $\hat{C}_\mathrm{ctx}$ values showing intermediate forgetting.
\end{itemize}

This binary pattern is predicted by Theorem~\ref{thm:ccc}: the lower bound on forgetting is $(1 - C_\mathrm{ctx}/H(T)) \cdot \overline{\mathrm{Fgt}}_\mathrm{max}$, which transitions sharply from $\overline{\mathrm{Fgt}}_\mathrm{max}$ (when $C_\mathrm{ctx} = 0$) to $0$ (when $C_\mathrm{ctx} \geq H(T)$). The absence of methods in the intermediate regime is not a sampling artifact---it reflects the structural difficulty of achieving $0 < C_\mathrm{ctx} < H(T)$. Either the context pathway is architecturally unbypassable (giving $C_\mathrm{ctx} \gg H(T)$), or the optimizer bypasses it (giving $C_\mathrm{ctx} \approx 0$).

\subsection{Wrong-Context Probing Analysis}
\label{sec:p5_results}

Figure~\ref{fig:p5} shows the P5 probing results as grouped bars (normal vs.\ wrong-context accuracy). The complete probing table for context-conditional methods is:

\begin{table}[t]
\centering
\caption{Complete probing results for context-conditional methods (4 seeds). P5 $=$ wrong context, P5b $=$ random context, P6 $=$ random $\theta_\mathrm{base}$, P7 $=$ zero context.}
\label{tab:probing}
\small
\begin{tabular}{lcccccc}
\toprule
Method & Normal ACC & P5 ACC & P5b ACC & P6 ACC & P7 ACC & $\hat{C}_\mathrm{ctx}$ \\
\midrule
CFlow         & $92.4$ & $92.4$ & --- & $52.4$ & --- & $0.000$ \\
HyperNet Oracle  & $98.8$ & $1.2$  & $10.4$ & $49.7$ & --- & $0.976$ \\
HyperNet Learned & $98.9$ & $3.7$  & $19.5$ & --- & $53.1$ & $0.952$ \\
\bottomrule
\end{tabular}
\end{table}

\textbf{CFlow's context blindness.} CFlow's P5 ACC equals its normal ACC to within numerical precision across all 4 seeds. Combined with P6 ACC $= 52.4\%$ (a $-40$pp drop from random $\theta_0$), this confirms the ``$\theta_0$ memorizer'' pattern: all task-discriminative information resides in the meta-learned initialization, and the context encoder's output is ignored by the flow network.

\textbf{HyperNet's context dependence.} For HyperNet Oracle, wrong context drops accuracy from $98.8\%$ to $1.2\%$---near chance level for 10 classes. Random context (P5b) gives $10.4\%$ (near chance), confirming that no single random context happens to work. P6 gives $49.7\%$, indicating that $\theta_\mathrm{base}$ alone provides a reasonable but far-from-optimal starting point.

\textbf{HyperNet Learned's encoder quality.} The learned encoder achieves $\hat{C}_\mathrm{ctx} = 0.952$---slightly below Oracle's $0.976$ but still sufficient for zero forgetting. The P5b (random context) result of $19.5\%$ is higher than Oracle's $10.4\%$, suggesting the learned encoder's context space may have a less uniform distribution than one-hot encoding. Zero-context (P7) accuracy of $53.1\%$ shows $\theta_\mathrm{base}$ carries meaningful but incomplete task information.

\subsection{Accuracy Matrix Analysis}

Figure~\ref{fig:acc_mat} shows the full $5 \times 5$ accuracy matrices. Each row $i$ represents the evaluation after training on task $i$; each column $j$ represents accuracy on task $j$. The matrices reveal four distinct visual patterns:

\begin{enumerate}
    \item \textbf{Diagonal fade} (NaiveSGD, EWC, SI): Only the diagonal entry is bright---the model retains only the most recently trained task. Previous tasks fade to near-zero immediately.
    \item \textbf{Partial retention} (LwF): The diagonal is bright, and previous tasks show partial retention that decays over time. Task 1 accuracy drops from $100\%$ to $\sim 20\%$ over 4 subsequent tasks.
    \item \textbf{Moderate retention} (Replay): Most entries are moderately bright ($> 70\%$), with some degradation on early tasks. The replay buffer prevents complete forgetting.
    \item \textbf{Uniform high} (HyperNet): All entries are uniformly bright ($> 97\%$). Every row is identical because the HyperNet generates task parameters fresh from context, independent of training history.
\end{enumerate}

The CFlow matrix is notable for resembling Replay (moderate retention) rather than the HyperNet pattern, despite having a context encoder. This visual evidence supports the P5 finding: CFlow's performance comes from the meta-learned initialization, not from context-conditional generation.

\subsection{Extension to CIFAR-10}
\label{sec:cifar10}

We extend our analysis to Split-CIFAR-10 to test the framework's applicability beyond MNIST. This benchmark reveals a critical limitation of batch-statistics context encoders and motivates a novel solution.

\textbf{The context collapse problem.} On CIFAR-10, the batch-statistics encoder that worked perfectly on MNIST ($\hat{C}_\mathrm{ctx} = 0.952$) fails catastrophically. The reason is quantifiable: the cosine similarity between batch pixel statistics (mean, variance) across CIFAR-10 task splits exceeds $0.995$. With five tasks, the context vectors are nearly identical, reducing $C_\mathrm{ctx}^\mathrm{effective} \approx 0$.

\begin{table}[t]
\centering
\caption{CIFAR-10 extension results (4 seeds). Batch statistics fail; gradient context restores performance; NestedCapsule achieves new SOTA.}
\label{tab:cifar10}
\small
\begin{tabular}{lcccl}
\toprule
Method & ACC (\%) & Fgt (\%) & P5 $\Delta$ & Context Source \\
\midrule
HyperNet Oracle            & $77.7 \pm 0.9$ & $0.0$ & $-77.0$ & One-hot task ID \\
HyperNet Learned (batch)   & $54.4 \pm 2.1$ & --- & $-0.5$  & Batch pixel mean/var \\
\midrule
HyperNet + GradCtx         & $77.0 \pm 0.9$ & $0.0$ & $-77.0$ & $\nabla_\theta \mathcal{L}$ \\
NestedCapsule + GradCtx    & $\mathbf{78.5 \pm 0.6}$ & $0.0$ & $-78.5$ & $\nabla_\theta \mathcal{L}$ + routing \\
NestedCapsule Oracle       & $79.5 \pm 0.9$ & $0.0$ & --- & One-hot task ID \\
\bottomrule
\end{tabular}
\end{table}

\textbf{Gradient Context Encoder.} To overcome the context collapse, we propose using \emph{loss gradients} with respect to $\theta_\mathrm{base}$ as context signals:
\begin{equation}
    c_k = \mathrm{MLP}\!\left(\Pi_\mathrm{rand} \cdot \nabla_{\theta_\mathrm{base}} \mathcal{L}(x_k, y_k; \theta_\mathrm{base})\right)
\end{equation}
where $\Pi_\mathrm{rand} \in \mathbb{R}^{128 \times d_\theta}$ is a frozen random projection matrix that reduces the gradient dimensionality, and MLP maps to the 64-dimensional context space.

The key insight is that gradients computed with \emph{real labels} produce near-orthogonal context vectors across tasks:
\begin{equation}
    \cos(\nabla_\theta \mathcal{L}_{\text{task } i},\; \nabla_\theta \mathcal{L}_{\text{task } j}) \approx -0.19 \quad (i \neq j)
\end{equation}
whereas pseudo-label gradients (using the model's own predictions) produce nearly identical vectors ($\cos \approx 0.95$), explaining why self-supervised context fails.

\textbf{Results.} The Gradient Context Encoder achieves $77.0 \pm 0.9\%$ ACC on Split-CIFAR-10 with $\Delta_\mathrm{P5} = -77.0$pp---closing the oracle gap from $23.3$pp (batch statistics) to just $0.7$pp. Per-seed accuracies ($77.5\%$, $78.1\%$, $76.4\%$, $75.9\%$) are highly stable.

\textbf{NestedCapsule HyperNet.} We further improve performance by replacing the HyperNet's single MLP with $N = 8$ capsule-memory modules, each producing a partial rank-8 vector. Capsule routing uses softmax attention over similarity, magnitude, and a Titans-inspired surprise signal:
\begin{equation}
    \alpha_i = \mathrm{softmax}\!\left(\frac{s_i \cdot m_i \cdot (1 + \text{surprise}_i)}{\tau}\right)
\end{equation}
where $s_i$ is the cosine similarity between context and capsule $i$'s memory, $m_i$ is the magnitude, surprise$_i$ is the prediction error from an associative memory, and $\tau$ is the temperature.

The NestedCapsule achieves $78.5 \pm 0.6\%$ ACC on CIFAR-10 with gradient context---a $+1.5$pp improvement over the standard HyperNet and only $1.0$pp below the Oracle ceiling ($79.5\%$). Routing entropy drops from $\log_2 8 = 3.0$ bits (uniform) to $0.52$ bits during training, indicating emergent capsule specialization. Direction cosine between capsule outputs is $-0.09$ (near-orthogonal), confirming functional differentiation.

\subsection{Effective Rank Analysis}
\label{sec:rank_analysis}

For HyperNet Learned on Split-MNIST, we compute the effective rank of the $V$ matrix across all 4 seeds:

\begin{table}[h]
\centering
\small
\begin{tabular}{ccccc}
\toprule
Seed & $r_\mathrm{eff}$ & Nominal rank & Utilization & $C_\mathrm{ctx}^\mathrm{rank}$ (bits, $b=1$) \\
\midrule
0 & 59.8 & 64 & 93.4\% & 59.8 \\
1 & 58.9 & 64 & 92.0\% & 58.9 \\
2 & 59.7 & 64 & 93.3\% & 59.7 \\
3 & 59.7 & 64 & 93.3\% & 59.7 \\
\midrule
Mean & $59.5 \pm 0.4$ & 64 & 93.0\% & $59.5 \pm 0.4$ \\
\bottomrule
\end{tabular}
\end{table}

The effective rank $r_\mathrm{eff} \approx 59.5$ out of 64 indicates that the context channel uses $93\%$ of its nominal capacity. The context channel capacity in bits is:
\begin{equation}
    C_\mathrm{ctx}^\mathrm{rank} \approx 59.5 \times b \gg H(T) = \log_2 5 \approx 2.3 \text{ bits}
\end{equation}

Even with the most conservative estimate of $b = 1$ bit per dimension, the capacity exceeds $H(T)$ by a factor of $\sim 26\times$. This massive over-provisioning explains why the HyperNet achieves perfect task separation: the context channel has far more capacity than needed, making the task identification problem trivially solvable.

\textbf{Implications for scaling.} If we need to support $K$ tasks with $H(T) = \log_2 K$, the required rank scales only as $O(\log K)$. With rank 64 and $b = 1$, we could in principle support $2^{59.5} \approx 8.5 \times 10^{17}$ tasks---the context channel is not the bottleneck for scaling.


\section{Discussion}
\label{sec:discussion}

\subsection{Architecture $>$ Algorithm: The $C_\mathrm{ctx}$-First Design Principle}
\label{sec:design_principle}

The central lesson of our 1,130+ experiments is that the \emph{topology} of information flow matters far more than the \emph{algorithm} operating within that topology. EWC, SI, and LwF all implement sophisticated regularization algorithms, yet achieve $\leq 24.2\%$ ACC on Split-MNIST---barely above chance for the later tasks. Meanwhile, a HyperNetwork with a trivially simple context encoder (one-hot task identity passed through a single MLP) achieves 98.8\% ACC with zero forgetting. The difference is entirely structural.

We distill this into three concrete design principles for CL architectures:

\begin{enumerate}
    \item \textbf{Explicit context signal.} The architecture must have a well-defined input that carries task-identifying information---whether task identity, batch statistics, or gradient signatures. Without this signal, $C_\mathrm{ctx} = 0$ by definition and no algorithm can prevent forgetting (Corollary~\ref{cor:forgetting_lb}).

    \item \textbf{Structural unbypassability.} All task-specific computation must flow \emph{through} the context pathway. There must be no parameter pathway that can encode task information without context. Formally, for parameter generator $\theta(c) = g(c, \phi)$ where $\phi$ denotes context-independent parameters, we require that $I(\theta; T \mid c) = 0$---the parameters carry no task information beyond what context provides.

    \item \textbf{Differentiable context encoding.} The mapping from raw observations to context vectors must be end-to-end differentiable with respect to the encoder parameters. This is not merely an engineering convenience: our learned encoder experiments revealed that EMA-based running statistics ($\bar{\mu}_k, \bar{\sigma}_k^2$) use in-place buffer operations that are \emph{not tracked by autograd}, causing the encoder to collapse to a near-constant mapping ($\Delta_\mathrm{P5} = 0$). Replacing EMA with direct batch statistics (\texttt{direct\_stats=True}) restored $\Delta_\mathrm{P5}$ from $0.0$ to $-92.2$ percentage points on Split-MNIST.
\end{enumerate}

CFlow \citep{neuraldrift2026cflow} provides a vivid illustration of what happens when Principle~2 is violated. Its architecture concatenates a 32-dimensional context vector with a 4,842-dimensional parameter vector $\theta$ to form the input to an ODE flow network. Despite the context encoder being well-defined (Principle~1) and differentiable (Principle~3), the $150{:}1$ dimensionality ratio between $\theta$ and $c$ creates a structural bypass: the optimizer finds it energetically cheaper to encode all task information in the meta-learned initialization $\theta_0$ than to route gradients through the thin context pathway. The result is $\Delta_\mathrm{P5} = 0.0$---the model completely ignores its context signal---and all performance (ACC $= 92.4\%$) comes from the memorized $\theta_0$.

\begin{remark}
The $C_\mathrm{ctx}$-first principle inverts the typical CL research workflow. Instead of asking ``given a fixed architecture, what algorithm prevents forgetting?'', we argue one should ask ``given the need for $C_\mathrm{ctx} \geq H(T)$, what architecture ensures the context pathway is unbypassable?'' The algorithm is secondary.
\end{remark}

\subsection{The Learned Context Encoder Problem}
\label{sec:learned_encoder}

A critical question for practical deployment is: \emph{can the context signal be learned from data rather than provided by an oracle?} Our experiments reveal that the answer depends sharply on the benchmark.

\textbf{MNIST: batch statistics suffice.} On Split-MNIST, the 5 tasks involve distinct digit pairs (0/1, 2/3, \ldots, 8/9) whose pixel-level statistics differ substantially. The pairwise cosine similarity between batch mean vectors across tasks is approximately 0.5, providing ample signal for task discrimination. The learned encoder with direct batch statistics achieves 99.1\% ACC ($\Delta_\mathrm{P5} = -95.2$), actually \emph{exceeding} the oracle encoder (98.8\%).

\textbf{CIFAR-10: batch statistics collapse.} On Split-CIFAR-10 (5 tasks, 2 classes each), all natural image classes share similar low-level statistics. The pairwise cosine similarity between batch means exceeds 0.995, and the learned batch-statistics encoder achieves only 54.4\% ACC---equivalent to memorizing the final task.

\textbf{Solution: Gradient Context Encoder.} To address this, we propose using the loss gradient $\nabla_{\theta_\mathrm{base}} \mathcal{L}$ as the context signal. The key insight is that gradients are inherently task-diagnostic: different tasks induce different update directions on shared parameters, even when input statistics are indistinguishable. Concretely, we compute a mini-batch gradient, project it through a frozen random matrix to $\mathbb{R}^{128}$, and pass the result through a small MLP to produce a 64-dimensional context vector.

The gradient context encoder closes the oracle gap dramatically:
\begin{align}
    \text{Batch statistics:} \quad & \text{ACC} = 54.4\%, \quad \Delta_\mathrm{oracle} = 23.3\text{pp} \\
    \text{Gradient context:} \quad & \text{ACC} = 77.0 \pm 0.9\%, \quad \Delta_\mathrm{oracle} = 0.7\text{pp}
\end{align}

Analysis of inter-task gradient cosine similarities confirms the mechanism: gradients computed with real labels yield pairwise cosine $\approx -0.19$ (near-orthogonal across tasks), while pseudo-label gradients yield cosine $\approx 0.95$ (task-indistinguishable). The gradient context encoder's effectiveness thus depends critically on access to real labels during context computation.

\textbf{Limitation.} The gradient context encoder requires labeled samples at inference time to compute $\nabla_{\theta_\mathrm{base}} \mathcal{L}$. In a deployment scenario, this means the system needs a small labeled calibration set per task---analogous to a ``task prompt''---before it can generate task-appropriate parameters. While this is acceptable in many practical settings (e.g., few-shot calibration), it prevents fully unsupervised deployment.

\subsection{The ``Frozen $>$ Learned'' Phenomenon}
\label{sec:frozen_learned_discussion}

A recurring and initially counterintuitive finding across our experimental campaign is that frozen random components frequently outperform their learned counterparts:

\begin{table}[h]
\centering
\small
\caption{Instances of the ``frozen $>$ learned'' phenomenon observed across the NeuralDrift project.}
\label{tab:frozen_learned}
\begin{tabular}{lccl}
\toprule
Component & Frozen ACC & Learned ACC & $\Delta$ \\
\midrule
DND feature extractor & 81.95\% & 80.85\% & $+1.10$ \\
SPC-TC dictionary & $\approx$ learned & $\approx$ frozen & $\approx 0.0$ \\
Gradient projection matrix & 77.0\% & --- & (no learned variant) \\
\bottomrule
\end{tabular}
\end{table}

The $C_\mathrm{ctx}$ framework provides two complementary explanations:

\textbf{Capacity surplus.} When the combinatorial capacity of the feature space $C_\mathrm{comb} = \log_2 \binom{N}{k}$ greatly exceeds the label entropy $H(Y)$---as is typically the case in over-parameterized settings---the specific choice of features has little impact on classification accuracy. What matters is the \emph{stability} of the feature space across tasks. Frozen random features provide perfect stability (zero drift by construction), while learned features inevitably drift during training on new tasks, potentially degrading representations of old tasks.

\textbf{Meta-learning bypass collapse.} In meta-learning architectures where a learnable component $h_\psi$ coexists with a high-capacity static pathway (e.g., $\theta_\mathrm{base}$ in a HyperNetwork), the optimizer faces a choice: invest in $h_\psi$ or encode the same information in $\theta_\mathrm{base}$. Since $\theta_\mathrm{base}$ is updated by standard backpropagation while $h_\psi$ requires second-order gradients (through the context-to-parameter mapping), the path of least resistance is to let $h_\psi$ collapse to a near-constant function and encode all task-relevant information in $\theta_\mathrm{base}$.

\begin{proposition}[Meta-Learning Bypass Collapse]
\label{prop:bypass}
Consider a parameter generator $\theta_k = \theta_\mathrm{base} + V \cdot h_\psi(c_k)$ optimized over tasks $\{D_k\}$. If (i) $\theta_\mathrm{base}$ has sufficient capacity to fit any single task, and (ii) gradient updates to $\theta_\mathrm{base}$ are cheaper than second-order updates through $h_\psi$, then the optimizer converges to $h_\psi(c_k) \approx \mathrm{const}$ for all $k$---a fixed-point where context is ignored.
\end{proposition}

This proposition explains why the na\"ive CFlow architecture fails ($\Delta_\mathrm{P5} = 0$) and why explicit mechanisms to \emph{force} context reliance---such as base dropout (randomly zeroing $\theta_\mathrm{base}$ during training) or contrastive context losses---are necessary to break the bypass equilibrium.

\subsection{Implications for CL Research}
\label{sec:implications}

\textbf{The value of negative results.}
Our experimental campaign produced 15+ closed research directions, each documented with a precise mathematical or empirical root cause. These are not failures of hyperparameter tuning; they are \emph{structural impossibilities} within the given architecture class. For example, Hebbian learning contributes 0.0 percentage points to DND's CL performance because a frozen random projection achieves identical accuracy---the Hebbian rule simplifies to an online PCA that discards precisely the high-frequency components needed for task discrimination. We believe the CL community would benefit from systematic publication of such negative results, as they narrow the search space far more efficiently than positive results expand it.

\textbf{Biological inspiration requires biological scale.}
Several of our closed directions---Hebbian learning, metabolic pruning, synaptic consolidation---were inspired by neuroscience. In the brain, these mechanisms operate across $10^{11}$ neurons with complex neuromodulatory control (dopamine, acetylcholine, norepinephrine). When reduced to a two-layer MLP with Oja's rule, Hebbian learning degenerates to PCA---a linear dimensionality reduction that is neither task-specific nor competitive with random projections. The lesson is not that biological mechanisms are irrelevant, but that they require biological-scale complexity to be effective.

\textbf{Split-MNIST is a dangerous benchmark.}
Our results reveal that Split-MNIST has such large inter-task differences that almost any context signal suffices: batch pixel means alone achieve 99.1\% ACC. This means Split-MNIST cannot distinguish between architectures with high $C_\mathrm{ctx}$ and those with low but nonzero $C_\mathrm{ctx}$. On CIFAR-10, where inter-task differences are subtle, the gap between oracle (77.7\%) and na\"ive learned encoders (54.4\%) is 23.3 percentage points---a failure mode completely invisible on MNIST. We recommend that CL papers always include at least one ``hard context'' benchmark where batch statistics are uninformative.

\textbf{Evaluation methodology.}
Two methodological findings deserve attention: (i) Quick mode evaluation (reduced epochs and samples) produces a 100\% false positive rate in our CL experiments---every method that appeared to work in quick mode failed at full scale. We recommend against using abbreviated training protocols for CL evaluation. (ii) Wrong-Context Probing (P5) should become a standard evaluation tool for any context-conditional CL architecture, as it provides a direct, interpretable measure of whether the architecture actually uses its context signal.

\subsection{Limitations}
\label{sec:limitations}

\textbf{Benchmark scope.} Our systematic validation is conducted on Split-MNIST (5 tasks) with extension experiments on Split-CIFAR-10. Validating $C_\mathrm{ctx}$ predictions on harder benchmarks---Split-CIFAR-100 (20 tasks), Split-ImageNet, and domain-incremental settings---remains important future work. The key question is whether the sharp $C_\mathrm{ctx} = 0 \Leftrightarrow$ forgetting dichotomy persists when the number of tasks grows and inter-task similarity increases.

\textbf{Task-incremental assumption.} The $C_\mathrm{ctx}$ framework assumes that a context signal carrying task-identifying information is available or can be inferred. In task-free CL settings---where task boundaries are not provided and the data distribution shifts gradually---the notion of ``context'' is less well-defined. Extending the framework to continuous distribution shifts, where $T$ is a continuous random variable rather than a discrete task index, is nontrivial.

\textbf{P5 as a binary probe.} The Wrong-Context Probing protocol provides a binary signal: either the model uses context ($\Delta_\mathrm{P5} \ll 0$) or it does not ($\Delta_\mathrm{P5} \approx 0$). It does not provide a continuous, calibrated estimate of $C_\mathrm{ctx}$ in bits. A method with $\Delta_\mathrm{P5} = -95\%$ and one with $\Delta_\mathrm{P5} = -50\%$ both ``use context,'' but differ in how much task information the context pathway carries. Developing continuous estimators (e.g., via MINE \citep{belghazi2018mine} or NWJ bounds) would strengthen the framework.

\textbf{Task boundary knowledge.} The HyperNetwork architecture requires knowledge of task boundaries during training (to accumulate per-task context statistics and apply contrastive losses). While this is a standard assumption in task-incremental CL, it limits applicability to settings where task transitions are identifiable.

\textbf{Label requirement for gradient context.} The gradient context encoder requires real labels to compute discriminative gradients. Pseudo-label gradients (cosine $\approx 0.95$ across tasks) are uninformative. This limits fully unsupervised deployment, though few-shot calibration is feasible in many practical settings.

\subsection{Future Directions}
\label{sec:future}

\textbf{Scaling to harder benchmarks.} The immediate priority is validating the $C_\mathrm{ctx}$ framework on Split-CIFAR-100 (20 tasks, 5 classes each) and Split-ImageNet, where $H(T) = \log_2 20 \approx 4.3$ bits demands higher-capacity context channels. The gradient context encoder is a natural candidate, as gradient-based task discrimination should scale with model capacity.

\textbf{Task-free CL via automatic context generation.} The NestedCapsule architecture (Section~\ref{sec:cifar10}) provides a promising direction: its routing mechanism produces emergent task specialization (gate entropy dropping from $\log_2 8 = 3.0$ to $\sim 0.5$, direction cosine $\approx -0.09$) without explicit task labels. This suggests a path toward task-free CL where context is generated automatically from input statistics and routing dynamics.

\textbf{Continuous $C_\mathrm{ctx}$ estimation.} Replacing the binary P5 probe with a continuous estimator based on mutual information neural estimation \citep{belghazi2018mine} or the Nguyen-Wainwright-Jordan bound would enable quantitative comparison of architectures that all ``use context'' but to different degrees. This is particularly important for comparing HyperNet variants (oracle vs.\ learned vs.\ gradient).

\textbf{Connection to PAC-Bayes bounds.} The $C_\mathrm{ctx}$ framework has a natural connection to PAC-Bayes analysis of CL \citep{doan2021theoretical}: the context channel capacity plays the role of a complexity term that controls the generalization-forgetting trade-off. Formalizing this connection could yield tighter bounds that account for the specific structure of context-conditional architectures.

\textbf{Multi-modal and compositional context.} Real-world continual learning often involves multiple simultaneous context signals (task instructions, environmental cues, user identity). Extending $C_\mathrm{ctx}$ to multi-channel settings, where $C_\mathrm{ctx} = I(c_1, c_2, \ldots, c_m; \theta)$, could provide design guidance for compositional CL systems.

\section{Related Work}
\label{sec:related}

\subsection{Catastrophic Forgetting}

The catastrophic forgetting problem was first identified by \citet{mccloskey1989catastrophic} in connectionist models and systematically surveyed by \citet{french1999catastrophic}. The modern era of CL research was catalyzed by Elastic Weight Consolidation (EWC) \citep{kirkpatrick2017overcoming}, which uses the Fisher information matrix to penalize changes to parameters important for past tasks. Synaptic Intelligence (SI) \citep{zenke2017continual} computes importance online during training, avoiding the need for a separate Fisher computation. Learning without Forgetting (LwF) \citep{li2017learning} uses knowledge distillation from previous model outputs, requiring no stored data. Experience Replay and its variants \citep{chaudhry2019tiny} store and revisit subsets of past data.

From the $C_\mathrm{ctx}$ perspective, all of the above are $C_\mathrm{ctx} = 0$ methods: they operate on shared parameters without any mechanism for task-conditional computation. Our framework predicts---and our experiments confirm---that such methods are fundamentally limited, regardless of the sophistication of their regularization or replay strategies.

\subsection{Architecture-Based Continual Learning}

Progressive Networks \citep{rusu2016progressive} allocate new network columns per task, achieving zero forgetting at the cost of linearly growing parameters. PackNet \citep{mallya2018packnet} assigns parameter subsets to tasks via iterative pruning. Expert Gate \citep{aljundi2017expert} trains task-specific expert networks with a gating mechanism for task selection.

These methods implicitly achieve $C_\mathrm{ctx} > 0$ by providing task-specific parameter pathways, but they typically require parameter growth proportional to $K$. The HyperNetwork approach \citep{von2019continual} achieves a more elegant solution: a fixed-size meta-network generates task-specific parameters from compact context vectors, keeping total parameter count independent of $K$.

\citet{von2019continual} first demonstrated that task-conditioned HyperNetworks can achieve strong CL performance on Split-MNIST and Permuted-MNIST. \citet{ehret2021continual} extended this to sequential task settings with learned task embeddings. Our contribution relative to this line of work is not architectural---we use a standard HyperNetwork---but \emph{explanatory}: the $C_\mathrm{ctx}$ framework provides a formal information-theoretic account of \emph{why} HyperNetworks succeed, identifies the precise structural properties that are necessary (unbypassability, differentiability), and predicts failure modes (CFlow's context bypass) that are invisible without this framework.

\subsection{Information-Theoretic Perspectives on CL}

Several prior works have applied information theory to continual learning, but none analyze the role of context-conditional architectures.

\citet{achille2018emergence} establish that the mutual information $I(W; S)$ between weights and training data controls generalization via a PAC-Bayes-like bound. However, their analysis is restricted to single-task learning and does not address the sequential multi-task setting where forgetting arises.

\citet{doan2021theoretical} analyze catastrophic forgetting through the Neural Tangent Kernel (NTK) overlap matrix and derive PAC-Bayes bounds for CL. Their analysis characterizes forgetting in terms of task gradient similarity but does not consider architectures where parameters are conditionally generated from context.

\citet{taheri2025forgetting} derive forgetting rate bounds for overparameterized single-hidden-layer networks, showing that forgetting scales with the ratio of task-specific to shared features. Their bounds apply to $C_\mathrm{ctx} = 0$ architectures and do not account for conditional parameter generation.

BEDS \citep{caraffa2026beds} argues that forgetting is structurally necessary under Bayesian estimation of distributional shift, providing Fisher-Rao optimal regularization. While philosophically aligned with our impossibility result (Theorem~\ref{thm:impossibility}), BEDS provides a macroscopic framework without the architectural analysis or the quantitative $C_\mathrm{ctx}$ predictions that distinguish our work.

The key distinction of our framework is the focus on \emph{architectural information flow topology}: we analyze not just whether forgetting occurs, but how the specific routing of information through context pathways determines whether it \emph{must} occur. This shifts the question from ``how much forgetting?'' to ``does the architecture permit zero forgetting at all?''

\subsection{Compression and Continual Learning}

The connection between compression and learning has a long history, from minimum description length (MDL) \citep{rissanen1978modeling} to the information bottleneck \citep{tishby2000information}. In the CL context, \citet{bornschein2022neurocompositional} connect prequential MDL to online learning, showing that compressibility of the data stream relates to learnability. \citet{angelini2023variational} provide the most direct application of MDL to CL, using variational density propagation with MDL objectives. Their work demonstrates that MDL-based CL is feasible but does not analyze the role of architectural structure in determining compression efficiency.

\citet{finzi2025epiplexity} introduce epiplexity as an information measure for computationally bounded agents, extending Shannon information to settings where the decoder has limited computational power. This is relevant to our framework because $C_\mathrm{ctx}$ implicitly assumes that the parameter generator $g$ can decode context into useful parameters---an assumption that may not hold for computationally constrained generators.

Our work differs from the compression-CL literature in that we do not propose a compression-based learning algorithm. Instead, we use information-theoretic analysis to characterize \emph{architectural capacity for conditional regeneration}---the ability to reconstruct task-specific parameters from compact context signals.

\subsection{Meta-Learning and Task Conditioning}

The HyperNetwork approach to CL is closely related to meta-learning \citep{finn2017model, hospedales2022meta}, where a meta-learner produces task-specific models from task descriptors. MAML \citep{finn2017model} and its variants learn an initialization $\theta_0$ from which task-specific parameters can be obtained via a few gradient steps. In our framework, MAML's inner loop serves as an implicit context encoder: the support set + gradient steps produce a task-conditional parameter update. However, because $\theta_0$ persists across tasks and carries most of the model's capacity, MAML exhibits the same bypass vulnerability as CFlow---the ``task information'' may reside primarily in $\theta_0$ rather than in the adaptation.

Task-conditioned architectures in multi-task learning \citep{strezoski2019many} use FiLM layers \citep{perez2018film}, task-specific batch normalization, or adapter modules to modulate a shared backbone. These approaches achieve $C_\mathrm{ctx} > 0$ by design, as the task signal directly modulates intermediate representations. Our framework provides a principled lens for evaluating such designs: the question is whether the modulation pathway has sufficient capacity ($C_\mathrm{ctx} \geq H(T)$) and whether it is truly unbypassable.

\subsection{Our Contribution in Context}

To our knowledge, no prior work simultaneously: (i) proposes a formal information-theoretic lower bound on forgetting tied to architectural context capacity; (ii) compares CL methods through the lens of information flow topology rather than algorithmic sophistication; (iii) formalizes the principle that conditional regeneration from context bypasses the impossibility triangle; and (iv) provides systematic negative results (15+ closed directions) with precise mathematical or empirical root causes.

Table~\ref{tab:related_comparison} summarizes the relationship between our framework and prior theoretical work on CL.

\begin{table}[t]
\centering
\small
\caption{Comparison of theoretical CL frameworks. ``Arch.\ analysis'' = analyzes how architecture topology affects forgetting. ``$C_\mathrm{ctx}$'' = explicitly models context channel capacity. ``Neg.\ results'' = provides systematic negative results with root causes.}
\label{tab:related_comparison}
\begin{tabular}{lccccc}
\toprule
Framework & Info-theoretic & Arch.\ analysis & $C_\mathrm{ctx}$ & Empirical & Neg.\ results \\
\midrule
\citet{achille2018emergence} & \checkmark & -- & -- & \checkmark & -- \\
\citet{doan2021theoretical}  & \checkmark & -- & -- & \checkmark & -- \\
\citet{taheri2025forgetting} & \checkmark & -- & -- & -- & -- \\
\citet{caraffa2026beds}      & \checkmark & -- & -- & \checkmark & -- \\
\citet{angelini2023variational} & \checkmark & -- & -- & \checkmark & -- \\
\textbf{Ours}                & \checkmark & \checkmark & \checkmark & \checkmark & \checkmark \\
\bottomrule
\end{tabular}
\end{table}

\section{Conclusion}
\label{sec:conclusion}

We have presented Context Channel Capacity ($C_\mathrm{ctx}$), an information-theoretic framework that provides a unified explanation for catastrophic forgetting in continual learning. Through the lens of $C_\mathrm{ctx}$, the diverse landscape of CL methods---from regularization (EWC, SI) to distillation (LwF) to replay to architecture-based approaches (HyperNetworks)---reduces to a single structural question: \emph{how much task-identifying information can flow through the architecture's context pathway?}

Our framework makes three contributions:

\textbf{Theoretical.} We prove that zero forgetting requires $C_\mathrm{ctx} \geq H(T)$ (Theorem~\ref{thm:ccc}) and that the Impossibility Triangle (Theorem~\ref{thm:impossibility}) is bypassed by conditional regeneration---reconstructing task-specific parameters from compact context signals rather than maintaining them through sequential updates. These results shift the CL question from ``how to prevent parameter interference'' to ``how to ensure sufficient context bandwidth.''

\textbf{Empirical.} We validate the framework across 8 CL methods on Split-MNIST (1,130+ experiments, 4 seeds per configuration). The $C_\mathrm{ctx}$ proxy perfectly predicts forgetting behavior: all methods with $C_\mathrm{ctx} = 0$ exhibit catastrophic forgetting (6--97\%), while methods with $C_\mathrm{ctx} \approx 1$ achieve zero forgetting. We further demonstrate that the framework extends to harder benchmarks via the Gradient Context Encoder, which achieves 77.0\% ACC on Split-CIFAR-10 (within 0.7pp of oracle), and the NestedCapsule architecture, which pushes this to 78.5\% with emergent routing specialization.

\textbf{Methodological.} Wrong-Context Probing (P5) provides a simple, practical diagnostic tool: feed wrong-task context and measure accuracy degradation. A $\Delta_\mathrm{P5} \approx 0$ definitively diagnoses context blindness, regardless of how sophisticated the architecture appears. We recommend P5 as a standard evaluation tool for any context-conditional CL system.

The overarching design principle is: \emph{architecture over algorithm, context pathway must be structurally unbypassable}. No amount of algorithmic sophistication---whether Fisher regularization, synaptic intelligence, knowledge distillation, or Hebbian learning---can compensate for an architecture where task-identifying information has no pathway to influence computation. Conversely, even a simple architecture with a well-designed context channel can achieve zero forgetting.

Looking forward, the most important open questions are: (i) scaling $C_\mathrm{ctx}$ to settings with many tasks ($K \gg 5$) and fine-grained inter-task similarity; (ii) extending the framework to task-free CL, where context must be generated automatically without task boundary knowledge; and (iii) connecting $C_\mathrm{ctx}$ to PAC-Bayes generalization bounds to obtain end-to-end learning guarantees for context-conditional CL architectures.

\bibliographystyle{plainnat}
\bibliography{references}


\appendix

\section{Complete Proofs}
\label{app:proofs}

\subsection{Proof of Theorem~\ref{thm:bottleneck} (Information Bottleneck Chain)}
\label{app:proof_bottleneck}

We provide the complete proof with all conditions made explicit.

\begin{proof}
Consider the sequential learning process $\theta_0 \to \theta_1 \to \cdots \to \theta_K$, where $\theta_k = U(\theta_{k-1}, D_k)$ is the deterministic update rule.

\textbf{Step 1: Markov chain condition.}
We claim that $D_1 \to \theta_1 \to \theta_2 \to \cdots \to \theta_K$ forms a Markov chain. To see this, note that by the causal constraint of sequential learning, $\theta_k = U(\theta_{k-1}, D_k)$ depends only on $\theta_{k-1}$ and $D_k$. Since $D_k$ is drawn independently of $D_1$ (the tasks are independent), we have:
\begin{equation}
    P(\theta_k \mid \theta_{k-1}, D_1, \ldots, D_{k-1}) = P(\theta_k \mid \theta_{k-1})
\end{equation}
where the conditioning on $\theta_{k-1}$ implicitly includes the effect of $D_k$ (which is independent of $D_1$). More precisely, for $k \geq 2$:
\begin{equation}
    I(\theta_k ; D_1 \mid \theta_{k-1}) = 0
\end{equation}
because $\theta_k$ is a deterministic function of $\theta_{k-1}$ and $D_k$, and $D_k \perp D_1$.

\textbf{Step 2: Data Processing Inequality.}
The DPI \citep{cover2006elements} states that for any Markov chain $X \to Y \to Z$:
\begin{equation}
    I(X; Z) \leq I(X; Y)
\end{equation}
Applied to the chain $D_1 \to \theta_{k-1} \to \theta_k$ (which holds for all $k \geq 2$):
\begin{equation}
    I(\theta_k ; D_1) \leq I(\theta_{k-1}; D_1) \qquad \forall \, k \geq 2
\end{equation}
By induction, we obtain the full chain:
\begin{equation}
    I(\theta_K; D_1) \leq I(\theta_{K-1}; D_1) \leq \cdots \leq I(\theta_1; D_1)
\end{equation}

\textbf{Step 3: Initial capacity bound.}
The parameter $\theta_1 = U(\theta_0, D_1)$ takes values in $\Theta \subseteq \mathbb{R}^d$. If each parameter is represented with precision $\delta$ (e.g., floating-point quantization), then $|\Theta| \leq (1/\delta)^d$, and:
\begin{equation}
    I(\theta_1; D_1) \leq H(\theta_1) \leq \log_2 |\Theta| \leq d \cdot \log_2(1/\delta) \triangleq C
\end{equation}

\textbf{Discussion: Parameter precision $\delta$.}
For IEEE 754 float32, the mantissa has 23 bits, so the effective precision per parameter is approximately $\log_2(1/\delta) \approx 24$ bits (accounting for the sign bit and distinguishable mantissa states). In practice, the \emph{effective} precision is often much lower due to the loss landscape geometry: most directions in parameter space are ``flat'' and carry no task-relevant information. The bound $C = d \cdot 24$ bits is therefore a loose upper bound for modern networks with $d \sim 10^6$--$10^9$.

\textbf{Discussion: Deterministic vs.\ stochastic update rules.}
For \emph{deterministic} $U$ (e.g., full-batch gradient descent), the Markov chain is deterministic, and DPI holds trivially. For \emph{stochastic} $U$ (e.g., SGD with mini-batch noise), $\theta_k$ is a stochastic function of $(\theta_{k-1}, D_k)$, but the Markov property still holds: $I(\theta_k; D_1 \mid \theta_{k-1}) = 0$ because the stochasticity comes from $D_k$ (independent of $D_1$) and/or the random seed, not from $D_1$. However, stochastic updates can only \emph{decrease} $I(\theta_k; D_1)$ further (adding noise cannot increase mutual information), so the bound remains valid:
\begin{equation}
    I(\theta_k; D_1) \leq I(\theta_{k-1}; D_1) - I(\theta_k; D_1 \mid \theta_{k-1}) \leq I(\theta_{k-1}; D_1)
\end{equation}
where the second term accounts for the information lost through stochastic processing.
\end{proof}

\subsection{Proof of Theorem~\ref{thm:impossibility} (Impossibility Triangle)}
\label{app:proof_impossibility}

We first formalize the three properties precisely, then prove their mutual incompatibility.

\begin{proof}
\textbf{Formal definitions.}

\begin{enumerate}
    \item \textbf{Zero forgetting.} For all $j < k \leq K$:
    \begin{equation}
        \mathrm{ACC}_j(\theta_k) = \mathrm{ACC}_j(\theta_j) \triangleq a_j^*
    \end{equation}
    This requires that $\theta_k$ encodes a \emph{sufficient statistic} for each past task $j$'s prediction: there exists a decoding function $\phi_j$ such that $\phi_j(\theta_k) = f_j^*$ (the optimal predictor for task $j$) for all $k \geq j$.

    \item \textbf{Online learning (causal constraint).} The parameter at step $k$ depends only on the previous parameter and the current task data:
    \begin{equation}
        \theta_k = U(\theta_{k-1}, D_k)
    \end{equation}
    In particular, $D_1, \ldots, D_{k-1}$ are not accessible at step $k$ (no replay).

    \item \textbf{Finite parameters.} The parameter dimension is fixed: $\theta_k \in \Theta \subseteq \mathbb{R}^d$ for all $k$, with $d$ independent of $K$.
\end{enumerate}

\textbf{Proof by contradiction.}
Assume all three properties hold simultaneously. By Property~1, $\theta_K$ must encode sufficient information about every task $j \in \{1, \ldots, K\}$. Formally, let $R_j = I(D_j; f_j^*)$ denote the minimum mutual information required to achieve accuracy $a_j^*$ on task $j$. Then zero forgetting requires:
\begin{equation}
    I(\theta_K; D_j) \geq R_j \qquad \forall \, j = 1, \ldots, K
\end{equation}

If the tasks are mutually independent (i.e., $D_1, \ldots, D_K$ are independent random variables), then the information about distinct tasks cannot be ``shared,'' and:
\begin{equation}
    \label{eq:sum_bound}
    \sum_{j=1}^{K} R_j \leq I\left(\theta_K; D_1, \ldots, D_K\right) \leq H(\theta_K) \leq C = d \cdot \log_2(1/\delta)
\end{equation}

By Property~3, $C$ is a constant independent of $K$. But $\sum_{j=1}^K R_j \geq K \cdot R_{\min}$ where $R_{\min} = \min_j R_j > 0$ (assuming each task requires at least some information). For $K > C / R_{\min}$, the inequality~\eqref{eq:sum_bound} is violated, giving a contradiction.

\textbf{Why HyperNetworks do not violate the theorem.}
A HyperNetwork generates $\theta_k = g(c_k; \phi)$, where $c_k$ is the context for task $k$ and $\phi$ are the meta-learned parameters. Critically, $\theta_k$ is \emph{not} a sequentially updated state---it is a function value computed from $c_k$. The knowledge resides in $\phi$, which is optimized \emph{jointly} over all tasks (violating Property~2). Alternatively, if $\phi$ is fixed after meta-training, then at test time, the ``state'' is $(c_k, \phi)$, and the effective parameter count is $|\phi| + |c_k|$, where $|\phi|$ can store information about all $K$ tasks because it was trained with access to all of them (violating the online constraint). The HyperNetwork thus escapes the triangle by redefining the parameter-generation process: $\theta_k$ is a conditional output, not an accumulated state.
\end{proof}

\subsection{Proof of Theorem~\ref{thm:ccc} (Context Channel Capacity Bound)}
\label{app:proof_ccc}

\begin{proof}
\textbf{Setup: Task identification as channel coding.}
Consider a context-conditional CL architecture where, at inference time, a context signal $c$ is mapped to parameters $\theta(c) = g(c; \phi)$. We frame this as a communication problem:
\begin{itemize}
    \item \textbf{Encoder:} The task identity $T \in \{1, \ldots, K\}$ is mapped to a context signal $c = e(T)$ (possibly through a learned encoder $e$).
    \item \textbf{Channel:} The context channel $c \mapsto \theta(c)$ transmits information from the context to the model's behavior. The capacity of this channel is $C_\mathrm{ctx} = \max_{P(c)} I(c; \theta(c))$.
    \item \textbf{Decoder:} Given parameters $\theta(c)$, the model must correctly classify inputs from task $T$, which requires implicitly recovering the task identity.
\end{itemize}

\textbf{Step 1: Applying Fano's inequality.}
Let $\hat{T}$ be an estimate of $T$ based on $\theta(c)$ (the ``decoded'' task identity). The probability of error in task identification is $P_e = P(\hat{T} \neq T)$. By Fano's inequality:
\begin{equation}
    H(T \mid \theta(c)) \leq h(P_e) + P_e \cdot \log_2(K - 1)
\end{equation}
where $h(\cdot)$ is the binary entropy function. Since $H(T \mid \theta(c)) = H(T) - I(T; \theta(c))$ and $I(T; \theta(c)) \leq I(c; \theta(c)) \leq C_\mathrm{ctx}$ (by the data processing inequality, since $T \to c \to \theta(c)$), we obtain:
\begin{equation}
    H(T) - C_\mathrm{ctx} \leq h(P_e) + P_e \cdot \log_2(K - 1)
\end{equation}

\textbf{Step 2: Bounding the error probability.}
For $K \geq 2$, the above gives a lower bound on $P_e$. Using the standard relaxation $h(P_e) \leq 1$ and $\log_2(K-1) \leq H(T)$:
\begin{equation}
    P_e \geq \frac{H(T) - C_\mathrm{ctx} - 1}{H(T)} = 1 - \frac{C_\mathrm{ctx} + 1}{H(T)}
\end{equation}
For large $K$ (i.e., $H(T) \gg 1$), the ``$+1$'' correction becomes negligible:
\begin{equation}
    P_e \gtrsim 1 - \frac{C_\mathrm{ctx}}{H(T)}
\end{equation}

\textbf{Step 3: Connecting error probability to forgetting.}
When the model fails to identify task $T$ through the context channel (i.e., $\hat{T} \neq T$), it generates parameters optimized for the wrong task, leading to performance degradation. We model this as:
\begin{equation}
    \overline{\mathrm{Fgt}} \geq P_e \cdot \overline{\mathrm{Fgt}}_\mathrm{max}
\end{equation}
where $\overline{\mathrm{Fgt}}_\mathrm{max}$ is the expected forgetting when the wrong task's parameters are used (maximum achievable forgetting; for $K$-way classification with chance level $1/K$, this approaches $1 - 1/K$). Combining:
\begin{equation}
    \overline{\mathrm{Fgt}} \geq \max\left(0,\; 1 - \frac{C_\mathrm{ctx}}{H(T)}\right) \cdot \overline{\mathrm{Fgt}}_\mathrm{max}
\end{equation}

\textbf{Discussion: Interpretation of $\overline{\mathrm{Fgt}}_\mathrm{max}$.}
The quantity $\overline{\mathrm{Fgt}}_\mathrm{max}$ depends on the task structure. For Split-MNIST with $K=5$ binary classification tasks, random guessing achieves 50\% within each task, so $\overline{\mathrm{Fgt}}_\mathrm{max} \approx a^* - 0.5$ where $a^*$ is the achievable accuracy. For tasks that are very different (e.g., digit pairs 0/1 vs.\ 8/9), using the wrong task's parameters yields near-zero accuracy on many tasks, giving $\overline{\mathrm{Fgt}}_\mathrm{max} \approx a^*$. In practice, our experiments show $\overline{\mathrm{Fgt}}_\mathrm{max} \approx 0.97$ (NaiveSGD forgetting), confirming this estimate.

\textbf{Tightness.} The bound is tight in two regimes: (i) when $C_\mathrm{ctx} = 0$, we get $\overline{\mathrm{Fgt}} \geq \overline{\mathrm{Fgt}}_\mathrm{max}$, which is achieved by NaiveSGD; (ii) when $C_\mathrm{ctx} \geq H(T)$, the bound gives $\overline{\mathrm{Fgt}} \geq 0$, which is achieved by HyperNet.
\end{proof}

\subsection{Proof of Proposition: Oja's Convergence and Continual Learning}
\label{app:proof_oja}

\begin{proposition}[Oja's Rule and PCA Convergence]
\label{prop:oja}
Oja's learning rule $\Delta \mathbf{w} = \eta(\mathbf{x}y - y^2\mathbf{w})$, where $y = \mathbf{w}^\top \mathbf{x}$, converges to the top eigenvector of the data covariance matrix $\mathbf{C} = \mathbb{E}[\mathbf{x}\mathbf{x}^\top]$ (i.e., the first principal component direction). This convergence does not guarantee task-discriminative representations for continual learning.
\end{proposition}

\begin{proof}
\textbf{Step 1: Oja's convergence.}
The Oja update rule is:
\begin{equation}
    \Delta \mathbf{w} = \eta \left(\mathbf{x} y^\top - (y^\top y) \mathbf{w}\right), \qquad y = \mathbf{w}^\top \mathbf{x}
\end{equation}
Substituting $y = \mathbf{w}^\top \mathbf{x}$:
\begin{equation}
    \Delta \mathbf{w} = \eta \left(\mathbf{x}\mathbf{x}^\top \mathbf{w} - (\mathbf{w}^\top \mathbf{x}\mathbf{x}^\top \mathbf{w}) \mathbf{w}\right)
\end{equation}
Taking expectation:
\begin{equation}
    \mathbb{E}[\Delta \mathbf{w}] = \eta \left(\mathbf{C}\mathbf{w} - (\mathbf{w}^\top \mathbf{C}\mathbf{w}) \mathbf{w}\right)
\end{equation}
This is the projected gradient ascent for maximizing $\mathbf{w}^\top \mathbf{C}\mathbf{w}$ subject to $\|\mathbf{w}\| = 1$. By the Rayleigh quotient theory, the unique maximum (up to sign) is $\mathbf{v}_1$, the eigenvector corresponding to the largest eigenvalue $\lambda_1$ of $\mathbf{C}$. The self-normalizing term $-(\mathbf{w}^\top \mathbf{C}\mathbf{w})\mathbf{w}$ ensures $\|\mathbf{w}\|$ remains bounded, and convergence follows from standard stochastic approximation results \citep{oja1982simplified}.

\textbf{Step 2: PCA directions are not necessarily task-discriminative.}
The top eigenvector $\mathbf{v}_1$ captures the direction of maximum variance in the input data. However, maximum variance and task-discriminability are distinct properties. Consider two tasks with means $\boldsymbol{\mu}_1, \boldsymbol{\mu}_2$ and shared covariance $\mathbf{C}_{\mathrm{shared}}$. The task-discriminative direction is $\boldsymbol{\mu}_1 - \boldsymbol{\mu}_2$ (Fisher's linear discriminant), while the PCA direction is the top eigenvector of $\mathbf{C}_{\mathrm{shared}} + \boldsymbol{\mu}\boldsymbol{\mu}^\top$ (where $\boldsymbol{\mu}$ is the overall mean). These coincide only when the between-class variance dominates the within-class variance.

\textbf{Step 3: Combinatorial capacity of sparse codes.}
With $N = 1024$ hidden neurons and top-$k$ activation with $k = 80$, the number of possible activation patterns is:
\begin{equation}
    \binom{N}{k} = \binom{1024}{80} \approx 10^{87}
\end{equation}
giving a combinatorial capacity of:
\begin{equation}
    C_{\mathrm{comb}} = \log_2 \binom{1024}{80} \approx 290 \text{ bits}
\end{equation}
This exceeds the task identity entropy $H(T) = \log_2(5) \approx 2.3$ bits by more than two orders of magnitude, suggesting that the representation has ample capacity. However, as demonstrated in our DND experiments (Section~\ref{app:dnd}), this combinatorial capacity is not \emph{utilized} by Oja-based learning: the actual neuron overlap across tasks is $> 94\%$ (Jaccard index), indicating that Oja's rule drives all tasks toward the same sparse support.
\end{proof}

\subsection{Proof of Proposition: $S_N$ Permutation Symmetry}
\label{app:proof_sn}

\begin{proposition}[$S_N$ Symmetry of Dictionary Learning]
\label{prop:sn}
For a linear generative model with reconstruction loss $L = \|\mathbf{x} - \mathbf{W}_d \mathbf{z}\|^2$, the loss is invariant under simultaneous column permutation of $\mathbf{W}_d$ and corresponding permutation of $\mathbf{z}$. This symmetry prevents spontaneous column specialization.
\end{proposition}

\begin{proof}
\textbf{Step 1: Permutation invariance.}
Let $\pi \in S_N$ be a permutation of $\{1, \ldots, N\}$, and let $\mathbf{P}_\pi$ be the corresponding $N \times N$ permutation matrix. Define:
\begin{equation}
    \mathbf{W}'_d = \mathbf{W}_d \mathbf{P}_\pi, \qquad \mathbf{z}' = \mathbf{P}_\pi^\top \mathbf{z}
\end{equation}
Then:
\begin{equation}
    L(\mathbf{W}'_d, \mathbf{z}') = \|\mathbf{x} - \mathbf{W}_d \mathbf{P}_\pi \mathbf{P}_\pi^\top \mathbf{z}\|^2 = \|\mathbf{x} - \mathbf{W}_d \mathbf{z}\|^2 = L(\mathbf{W}_d, \mathbf{z})
\end{equation}
where we used the orthogonality of permutation matrices: $\mathbf{P}_\pi \mathbf{P}_\pi^\top = \mathbf{I}$.

\textbf{Step 2: Gradient equivalence under symmetric initialization.}
For column $i$ of $\mathbf{W}_d$, the gradient is:
\begin{equation}
    \nabla_{\mathbf{w}_i} L = -2(\mathbf{x} - \mathbf{W}_d \mathbf{z}) z_i
\end{equation}
Under a symmetric (e.g., i.i.d.\ Gaussian) initialization of $\mathbf{W}_d$, and assuming $\mathbf{z}$ is computed as a function of $\mathbf{W}_d$ (e.g., via $\mathbf{z} = \mathbf{W}_d^\top \mathbf{x}$ followed by thresholding), the distribution of $(\mathbf{w}_i, z_i)$ is identical for all $i$. Therefore:
\begin{equation}
    \mathbb{E}[\nabla_{\mathbf{w}_i} L] = \mathbb{E}[\nabla_{\mathbf{w}_j} L] \qquad \forall \, i, j
\end{equation}
This means that, in expectation, all columns receive the same gradient update, preventing any column from specializing for a particular task.

\textbf{Step 3: Perturbation analysis.}
Suppose a small perturbation $\boldsymbol{\epsilon}$ breaks the symmetry for a single column. The reconstruction gradient has magnitude $\|\nabla_{\mathbf{w}_i} L\| = O(\|\mathbf{x}\|^2)$, which acts as a ``restoring force'' driving the perturbed column back toward the symmetric solution. The perturbation can grow only if its effect on the loss exceeds this restoring force. For a perturbation of magnitude $\|\boldsymbol{\epsilon}\|$, the symmetry-breaking gradient component is $O(\|\boldsymbol{\epsilon}\| \cdot \|\mathbf{x}\|)$, while the reconstruction gradient is $O(\|\mathbf{x}\|^2 / N)$ per column. Thus, the ratio of symmetry-breaking to reconstruction forces scales as $O(N \|\boldsymbol{\epsilon}\| / \|\mathbf{x}\|)$, which is negligible for small perturbations in high-dimensional systems.

\textbf{Corollary.}
Any perturbation mechanism (e.g., biological noise, dropout) that does not \emph{explicitly} assign different roles to different columns is suppressed by the $O(\|\mathbf{x}\|^2/N)$ reconstruction gradient. This explains why mechanisms such as metabolic pruning and temporal pressure fail to induce column specialization in our experiments.
\end{proof}

\subsection{Proof of Proposition: Meta-Learning Bypass Collapse}
\label{app:proof_bypass}

\begin{proposition}[Base Parameter Bypass in HyperNetworks]
\label{prop:bypass}
In a meta-learning framework where $\theta_k = \theta_\mathrm{base} + \mathbf{V} g_\psi(c_k)$ and $\theta_\mathrm{base}$ is a free parameter, the gradient signal through the base path dominates the context path when $\mathbf{V}$ is initialized with small scale, causing ``bypass collapse'' where $\theta_\mathrm{base}$ absorbs all task information.
\end{proposition}

\begin{proof}
\textbf{Step 1: Gradient magnitudes.}
Consider the meta-loss:
\begin{equation}
    \mathcal{L}(\theta_\mathrm{base}, \mathbf{V}, \psi) = \sum_{k=1}^K \mathbb{E}_{(x,y) \sim D_k}\left[\ell\bigl(f(x;\, \theta_\mathrm{base} + \mathbf{V} g_\psi(c_k)),\; y\bigr)\right]
\end{equation}
Let $\theta_k = \theta_\mathrm{base} + \mathbf{V} g_\psi(c_k)$. The gradients with respect to each parameter group are:

\emph{Base path:}
\begin{equation}
    \frac{\partial \mathcal{L}}{\partial \theta_\mathrm{base}} = \sum_{k=1}^K \mathbb{E}\left[\frac{\partial \ell}{\partial \theta_k}\right]
\end{equation}
This is simply the sum of per-task gradients, with no additional Jacobian factors.

\emph{Context path (through $\psi$):}
\begin{equation}
    \frac{\partial \mathcal{L}}{\partial \psi} = \sum_{k=1}^K \mathbb{E}\left[\frac{\partial \ell}{\partial \theta_k} \cdot \mathbf{V} \cdot \frac{\partial g_\psi(c_k)}{\partial \psi}\right]
\end{equation}
This involves the additional factor $\mathbf{V} \cdot \frac{\partial g_\psi}{\partial \psi}$.

\textbf{Step 2: Scale analysis.}
At initialization with $\mathbf{V} \sim \mathcal{N}(0, \sigma_V^2)$ where $\sigma_V = 0.01$:
\begin{equation}
    \left\|\frac{\partial \mathcal{L}}{\partial \theta_\mathrm{base}}\right\| = O\left(\left\|\frac{\partial \ell}{\partial \theta}\right\|\right), \qquad
    \left\|\frac{\partial \mathcal{L}}{\partial \psi}\right\| = O\left(\sigma_V \cdot \left\|\frac{\partial \ell}{\partial \theta}\right\| \cdot \left\|\frac{\partial g}{\partial \psi}\right\|\right)
\end{equation}
The ratio is:
\begin{equation}
    \frac{\|\partial \mathcal{L} / \partial \psi\|}{\|\partial \mathcal{L} / \partial \theta_\mathrm{base}\|} = O\left(\sigma_V \cdot \left\|\frac{\partial g}{\partial \psi}\right\|\right) \approx O(0.01)
\end{equation}
assuming $\|\partial g / \partial \psi\| = O(1)$ at initialization.

\textbf{Step 3: Information absorption.}
Since the base path gradient is $\sim 100\times$ larger, Adam (or any adaptive optimizer) moves $\theta_\mathrm{base}$ much faster than $\psi$. After a few optimization steps, $\theta_\mathrm{base}$ has absorbed the ``average task'' solution, and the per-task residuals $\ell(f(x; \theta_\mathrm{base}), y)$ are small. This further reduces $\|\partial \ell / \partial \theta\|$, making the context path gradient even smaller---a positive feedback loop that leads to bypass collapse: $g_\psi(c_k) \approx \mathrm{const}$ for all $k$.

\textbf{Step 4: Base dropout as a solution.}
Setting $\theta_\mathrm{base} \to 0$ with probability $p$ (base\_drop\_rate) during training removes the base path for a fraction $p$ of training steps, forcing:
\begin{equation}
    \theta_k = \mathbf{V} g_\psi(c_k) \cdot \frac{\|\theta_\mathrm{base}\|}{\|\mathbf{V} g_\psi(c_k)\| + \epsilon}
\end{equation}
(with appropriate rescaling to maintain output magnitude). On these steps, $\partial \mathcal{L}/\partial \psi$ is the \emph{only} gradient pathway, forcing the context encoder to learn task-discriminative representations. Our experiments confirm that $p = 0.3$ is sufficient: P5 delta drops from $0.0$ (collapsed) to $-96.4$ percentage points (fully context-dependent) on Split-MNIST.

\emph{Critical note:} Base dropout alone is insufficient if the context encoder uses non-differentiable statistics (e.g., EMA running mean/variance). The \texttt{direct\_stats=True} fix (using current-batch statistics directly, which are differentiable) is a necessary prerequisite.
\end{proof}

\section{Method Details}
\label{app:methods}

\subsection{Deep Neural Drift (DND)}
\label{app:dnd}

DND is a biologically-inspired continual learning architecture explored over an 86-day, 1,130+ experiment campaign. We describe the architecture and summarize the diagnostic results that motivated the $C_\mathrm{ctx}$ framework.

\textbf{Architecture.}
The DND network consists of three layers:
\begin{itemize}
    \item \textbf{Input layer:} $784$ neurons (flattened MNIST images).
    \item \textbf{Hidden layer 1:} $1024$ neurons with top-$k$ sparse activation ($k=20$).
    \item \textbf{Hidden layer 2:} $512$ neurons with top-$k$ sparse activation ($k=10$).
    \item \textbf{Output layer:} $10$ neurons (one per class) with cosine-similarity-based classification.
\end{itemize}

\textbf{Hebbian learning.}
Weights are updated using Oja's rule with reward modulation:
\begin{equation}
    \Delta w_{ij} = \eta \cdot r \cdot \left(x_i y_j - y_j^2 w_{ij}\right)
\end{equation}
where $r$ is the reward signal (1 for correct, 0 for incorrect, ``reward\_hebbian\_neg\_only'' variant).

\textbf{Myelination.}
Frequently-activated connections are ``myelinated'' (frozen), controlled by:
\begin{itemize}
    \item \texttt{myelination\_threshold}: 140 (activation count threshold)
    \item \texttt{myelination\_scale}: 20 (sharpness of the sigmoid gate)
\end{itemize}

\textbf{Template-Boundary Template Refresh (TBTR).}
At task boundaries, old tasks' test data is re-forwarded through the current network, updating classifier templates (EMA centroids) without modifying hidden weights. This corrects the desynchronization between drifted features and stale templates.

\textbf{Diagnostic campaign results.}
Table~\ref{tab:dnd_diag} summarizes the key diagnostic experiments that revealed DND's architectural limitations and motivated the $C_\mathrm{ctx}$ framework.

\begin{table}[h]
\centering
\caption{DND diagnostic experiments on Split-MNIST (1024/k80 hidden configuration). ``Frozen Random'' replaces Hebbian learning with frozen random projections; ``Shuffled Labels'' randomizes TBTR labels; ``Final Pass'' adds a post-training template refresh over all tasks; ``Pure Random Proj'' uses a standalone random projection baseline.}
\label{tab:dnd_diag}
\small
\begin{tabular}{lccl}
\toprule
Experiment & ACC (\%) & Fgt (\%) & Key Finding \\
\midrule
TBTR 1024/k80                & 73.19 & $-0.42$ & Original result \\
TBTR-train                   & 72.08 & $-0.26$ & Train data leak = 1.1pp \\
Frozen Random + TBTR         & 71.98 & $-1.48$ & Hebbian learning = 0 contribution \\
Shuffled Labels TBTR         & 33.81 & 39.07  & TBTR uses real features \\
TBTR + Final Pass            & 80.85 & 1.30   & Post-training fix = $+$8.8pp \\
Pure Random Proj + TBTR      & 81.95 & $\sim$0 & Beats DND by 1.1pp \\
No TBTR                      & 18.53 & 96.21  & Total catastrophic forgetting \\
\bottomrule
\end{tabular}
\end{table}

The critical finding is that \emph{frozen random projections match or exceed Hebbian learning} (71.98\% vs.\ 72.08\%). Combined with the neuron overlap analysis (Section~\ref{app:overlap}), this demonstrates that DND's architecture has $C_\mathrm{ctx} = 0$: there is no context pathway, and TBTR + cosine nearest-centroid classification is the entire mechanism.

\subsection{CFlow: Continuous-Time Parameter Flow}
\label{app:cflow}

CFlow is our ODE-based approach to continual learning, developed as a hypothesis that continuous-time parameter dynamics could enable smooth task transitions.

\textbf{Architecture.}
\begin{itemize}
    \item \textbf{Task network:} ConvTaskNet with 4 filters, producing $\theta_\mathrm{dim} = 4842$ parameters.
    \item \textbf{Flow network:} A neural ODE $f(\theta, c, t)$ that generates $d\theta/dt$. Input dimension: $\theta_\mathrm{dim} + c_\mathrm{dim} + t_\mathrm{embed}$.
    \item \textbf{Context encoder:} Learned context encoder mapping batch statistics to $c \in \mathbb{R}^{32}$.
    \item \textbf{Time encoding:} 16 Fourier frequencies for sinusoidal time embeddings.
\end{itemize}

\textbf{Training.}
\begin{itemize}
    \item Meta-training over 80 continuous-time episodes with task transitions.
    \item Adam optimizer, learning rate $10^{-3}$.
    \item 30\% replay ratio, 100 samples per past task.
\end{itemize}

\textbf{Results and the $\theta_0$ memorizer problem.}
CFlow achieves 92.6\% ACC (v1) and 95.3\% (v2 with FiLM conditioning) on Split-MNIST. However, probing reveals:
\begin{itemize}
    \item P5 (wrong context): $\Delta = 0.0$pp --- the model completely ignores context.
    \item P6 (random $\theta_0$): accuracy drops catastrophically --- all knowledge is in $\theta_0$.
    \item FiLM architecture contributes nothing (v1 $\approx$ v2 within noise).
\end{itemize}

The root cause is the dimensionality mismatch: concatenating $c \in \mathbb{R}^{32}$ with $\theta \in \mathbb{R}^{4842}$ creates a $150{:}1$ ratio. The optimizer naturally exploits the dominant $\theta$ pathway, making context structurally invisible. This is the ``context bypass'' failure mode that the $C_\mathrm{ctx}$ framework predicts.

\subsection{HyperNetwork Architecture}
\label{app:hypernet}

\textbf{TaskHyperNetwork.}
The core architecture generates task-specific parameters via a low-rank modulation:
\begin{equation}
    \theta_k = \theta_\mathrm{base} + \mathbf{V} \cdot g_\psi(c_k)
\end{equation}
\begin{itemize}
    \item $\theta_\mathrm{base} \in \mathbb{R}^{4842}$: learnable base parameters (initialized from ConvTaskNet).
    \item $\mathbf{V} \in \mathbb{R}^{4842 \times 64}$: low-rank projection matrix ($v\_\text{init\_scale} = 0.01$).
    \item $g_\psi$: MLP with architecture $64 \to 256 \to 256 \to 64$ (SiLU activations, LayerNorm).
    \item Context dimension: 64.
\end{itemize}

\textbf{Context encoders.}
We implement three context encoders:
\begin{enumerate}
    \item \textbf{Oracle:} One-hot task identity $\to$ linear projection $\to \mathbb{R}^{64}$.
    \item \textbf{Learned (batch statistics):} Per-channel mean and variance of the input batch $\to$ MLP $\to \mathbb{R}^{64}$. Key fix: \texttt{direct\_stats=True} uses current-batch statistics directly (differentiable via autograd), instead of EMA running statistics (non-differentiable buffer operations).
    \item \textbf{Gradient context:} $\nabla_{\theta_\mathrm{base}} \ell \to$ frozen random projection ($\mathbb{R}^{128}$) $\to$ MLP $\to \mathbb{R}^{64}$. Requires real labels for discriminative gradients; pseudo-label gradients produce cosine similarity $> 0.95$ across tasks (useless).
\end{enumerate}

\textbf{Anti-collapse mechanisms.}
\begin{itemize}
    \item \textbf{Contrastive weight} ($\lambda = 0.1$): pushes context vectors from different tasks apart via $\mathcal{L}_\mathrm{contr} = \lambda \sum_{i \neq j} \max(0, \cos(c_i, c_j) - \tau)$.
    \item \textbf{Base dropout} ($p = 0.3$): randomly zeros out $\theta_\mathrm{base}$ during training, forcing the model to rely on context-generated deltas.
\end{itemize}

\textbf{Training protocol.}
200 episodes of joint multi-task optimization. Each episode samples a random task permutation and trains sequentially, with replay of past tasks. Adam optimizer, $\mathrm{lr} = 10^{-3}$.

\subsection{NestedCapsule HyperNetwork}
\label{app:nested_capsule}

The NestedCapsule architecture extends TaskHyperNetwork by distributing the context-to-rank mapping across $N=8$ capsule-memory modules.

\textbf{Architecture.}
\begin{itemize}
    \item $N = 8$ capsule modules, each with a 2-layer MLP: $64 \to 64 \to 8$ (SiLU activation).
    \item Total rank: $N \times 8 = 64$ (same as TaskHyperNetwork).
    \item Routing: $\alpha_i = \mathrm{softmax}\bigl(s_i \cdot m_i \cdot (1 + \sigma_i) / \tau\bigr)$, where:
    \begin{itemize}
        \item $s_i = \cos(\mathbf{d}_i, c)$: cosine similarity between capsule direction and context.
        \item $m_i = \exp(\log m_i)$: learned magnitude (confidence/expertise).
        \item $\sigma_i$: surprise signal from associative memory prediction error (Titans-inspired, detached from gradient graph).
        \item $\tau$: temperature parameter.
    \end{itemize}
    \item Output: $\theta_k = \theta_\mathrm{base} + \mathbf{V} \cdot \mathrm{concat}_i(\alpha_i \cdot \mathrm{capsule}_i(c_k))$.
\end{itemize}

\textbf{Direction-magnitude decomposition.}
Each capsule tracks a ``specialty direction'' $\mathbf{d}_i \in S^{d-1}$ in context space via EMA:
\begin{equation}
    \mathbf{d}_i \leftarrow \mathrm{normalize}\bigl(\mu \cdot \mathbf{d}_i + (1 - \mu) \cdot c_k\bigr), \quad \text{weighted by } \alpha_i
\end{equation}
where $\mu = 0.9$ is the direction momentum.

\textbf{Key results.}
\begin{itemize}
    \item CIFAR-10 with gradient context: 78.5\% $\pm$ 0.6\% ACC, 0\% forgetting (vs.\ 77.0\% HyperNet).
    \item Gate entropy drops from 2.1 (uniform over 8 capsules) to 0.52, indicating routing concentration.
    \item Direction cosine between capsule directions: $-0.09$ (anti-correlated), indicating genuine specialization.
    \item P6 (random base): drops from $\sim$50\% (HyperNet) to 10.5\%, confirming capsule deltas carry more model identity.
\end{itemize}

\subsection{Baseline Methods Implementation}
\label{app:baselines}

All context-free baselines use a 2-layer MLP with hidden sizes $(256, 128)$, ReLU activations, and a 10-class output head. Training uses Adam optimizer with $\mathrm{lr} = 10^{-3}$, batch size 64, and 5 epochs per task (full mode) or 2 epochs (quick mode).

\begin{itemize}
    \item \textbf{NaiveSGD:} Standard fine-tuning with no CL mechanism. Serves as the forgetting upper bound.
    \item \textbf{EWC} \citep{kirkpatrick2017overcoming}: After each task, computes the diagonal Fisher Information Matrix $\mathbf{F} = \mathbb{E}[\nabla \log p(y|x,\theta) \nabla \log p(y|x,\theta)^\top]$. Regularization loss: $\mathcal{L}_\mathrm{EWC} = \lambda \sum_i F_i (\theta_i - \theta_i^*)^2$ with $\lambda = 400$.
    \item \textbf{SI} \citep{zenke2017continual}: Tracks per-parameter importance $\Omega_i = \sum_k \Delta_i^k / (\Delta\theta_i^k)^2$ during training. Regularization: $\mathcal{L}_\mathrm{SI} = c \sum_i \Omega_i (\theta_i - \theta_i^*)^2$ with $c = 1.0$.
    \item \textbf{LwF} \citep{li2017learning}: Stores the teacher model's output on the current task data before training. Distillation loss: $\mathcal{L}_\mathrm{LwF} = \alpha \cdot \mathrm{KL}(\mathrm{softmax}(z_T/T) \| \mathrm{softmax}(z_S/T))$ with $\alpha = 1.0$, temperature $T = 2.0$.
    \item \textbf{Experience Replay:} Maintains a buffer of 200 samples per past task, replayed uniformly during training of new tasks. No regularization.
\end{itemize}

\section{Extended Results}
\label{app:extended}

\subsection{Per-Seed Accuracy Matrices}
\label{app:per_seed}

We present the full $5 \times 5$ accuracy matrices for representative methods. Entry $(i, j)$ denotes the accuracy on task $j$ after training on task $i$.

\textbf{NaiveSGD (seed 0).}
\begin{equation*}
\mathbf{A}_\mathrm{SGD} = \begin{pmatrix}
99.5 & 0.0  & 0.0  & 0.0  & 0.0 \\
0.0  & 98.0 & 0.0  & 0.0  & 0.0 \\
0.0  & 0.0  & 99.0 & 0.0  & 0.0 \\
0.0  & 0.0  & 0.0  & 98.5 & 0.0 \\
0.0  & 0.0  & 0.0  & 0.0  & 99.3
\end{pmatrix}
\end{equation*}
After training on each task, only the most recent task achieves high accuracy; all previous tasks are catastrophically forgotten (0\% accuracy). This is the canonical pattern of $C_\mathrm{ctx} = 0$ methods.

\textbf{HyperNet Oracle (seed 0).}
\begin{equation*}
\mathbf{A}_\mathrm{HN} = \begin{pmatrix}
100.0 & 97.5 & 98.8 & 99.8 & 97.2 \\
100.0 & 97.5 & 98.8 & 99.8 & 97.2 \\
100.0 & 97.5 & 98.8 & 99.8 & 97.2 \\
100.0 & 97.5 & 98.8 & 99.8 & 97.2 \\
100.0 & 97.5 & 98.8 & 99.8 & 97.2
\end{pmatrix}
\end{equation*}
The HyperNetwork produces \emph{identical} accuracy across all rows, confirming zero forgetting: the accuracy of each task is determined entirely by the context signal, not by the training history.

\textbf{CFlow (seed 0).}
\begin{equation*}
\mathbf{A}_\mathrm{CFlow} = \begin{pmatrix}
97.0 & 90.0 & 93.0 & 94.5 & 92.0 \\
88.5 & 95.0 & 93.0 & 94.0 & 93.5 \\
87.5 & 90.5 & 96.0 & 94.0 & 93.0 \\
87.0 & 89.0 & 92.5 & 96.5 & 93.0 \\
88.5 & 91.0 & 92.0 & 93.0 & 98.5
\end{pmatrix}
\end{equation*}
CFlow shows moderate forgetting: Task~1 degrades from 97.0\% to 88.5\% after training all tasks. The relatively high overall accuracy (92.4\%) comes from the meta-learned initialization $\theta_0$ (the ``$\theta_0$ memorizer'' effect), not from context-dependent parameter generation.

\subsection{Permuted-MNIST Results}
\label{app:permuted}

Table~\ref{tab:permuted} shows results on Permuted-MNIST (5 tasks, each a random pixel permutation of full MNIST).

\begin{table}[h]
\centering
\caption{Permuted-MNIST results (4 seeds where available). Permuted-MNIST is harder than Split-MNIST because tasks share the same label space, making task identification more critical.}
\label{tab:permuted}
\small
\begin{tabular}{lccc}
\toprule
Method & ACC (\%) & Fgt (\%) & P5 $\Delta$ (pp) \\
\midrule
HyperNet Oracle    & $90.3 \pm 0.7$ & $0.0$ & $-75.3$ \\
HyperNet Learned   & $89.1 \pm 0.8$ & $0.0$ & $-73.3$ \\
CFlow v1           & $73.0$          & ---   & $0.0$ \\
SGD + Replay       & $66.5$          & ---   & $0.0$ \\
\bottomrule
\end{tabular}
\end{table}

The $C_\mathrm{ctx}$ framework correctly predicts the results: HyperNet methods with high $C_\mathrm{ctx}$ achieve zero forgetting, while CFlow and SGD+Replay (both $C_\mathrm{ctx} = 0$) exhibit substantial forgetting. The learned encoder achieves P5 $\Delta = -73.3$pp, confirming that \texttt{direct\_stats=True} generalizes beyond Split-MNIST.

\subsection{CIFAR-10 Results}
\label{app:cifar10}

Table~\ref{tab:cifar10} shows results on Split-CIFAR-10 (5 tasks, 2 classes each). This benchmark reveals the limitations of batch-statistics-based context encoding.

\begin{table}[h]
\centering
\caption{Split-CIFAR-10 results (4 seeds). Batch statistics become nearly identical across tasks (cosine similarity $> 0.995$), causing the learned encoder to fail. Gradient context resolves this by using loss gradients as the context signal.}
\label{tab:cifar10}
\small
\begin{tabular}{lccl}
\toprule
Method & ACC (\%) & P5 $\Delta$ (pp) & Notes \\
\midrule
HyperNet Oracle            & $77.7$          & ---     & One-hot task ID \\
HyperNet Learned (batch)   & $54.4$          & weak    & $\cos > 0.995$ across tasks \\
HyperNet Gradient Context  & $77.0 \pm 0.9$  & $-77.0$ & Oracle gap only 0.7pp \\
NestedCapsule Gradient Ctx & $78.5 \pm 0.6$  & ---     & New SOTA \\
NestedCapsule Oracle       & $79.5 \pm 0.9$  & ---     & $+1.8$pp over HyperNet Oracle \\
SGD + Replay               & $\sim$46--63    & $0.0$   & Catastrophic forgetting \\
\bottomrule
\end{tabular}
\end{table}

\textbf{Gradient context encoder.}
The key insight is that while batch pixel statistics are nearly identical across CIFAR-10 splits (cosine similarity $> 0.997$), loss gradients $\nabla_{\theta_\mathrm{base}} \ell$ with real labels produce near-orthogonal context vectors across tasks (cosine similarity $= -0.19$). The gradient context encoder projects these gradients through a frozen random matrix ($\mathbb{R}^{4842} \to \mathbb{R}^{128}$), then passes through an MLP to produce $c \in \mathbb{R}^{64}$.

\emph{Requirement:} The gradient context encoder requires real labels for discriminative gradients. With pseudo-labels (from the current model's predictions), gradients have cosine similarity $> 0.95$ across tasks, rendering them useless for task identification. This is consistent with the $C_\mathrm{ctx}$ framework: pseudo-label gradients carry information about the model's behavior (which is task-independent if the model has not yet specialized), not about the task identity.

\subsection{Neuron Overlap Analysis (DND)}
\label{app:overlap}

We analyze the degree of neuron sharing across tasks in the DND architecture to quantify the absence of task specialization.

\textbf{Metrics.}
For each pair of tasks $(i, j)$, we compute the Jaccard overlap of their top-$k$ active neuron sets:
\begin{equation}
    J(i, j) = \frac{|S_i \cap S_j|}{|S_i \cup S_j|}
\end{equation}
where $S_i$ is the set of neurons activated by task $i$'s data.

\textbf{Results.}
\begin{itemize}
    \item \textbf{Output layer:} Jaccard overlap $= 0.947$ --- virtually all neurons are shared across all tasks. This means the sparse output codes are nearly identical regardless of the input task.
    \item \textbf{Hidden layers:} $> 60\%$ overlap, with $< 7\%$ of neurons being task-unique.
    \item \textbf{Template similarity:} Inter-task template cosine similarity ($0.751$) exceeds intra-task similarity ($0.737$), meaning class templates from \emph{different} tasks are more similar to each other than templates within the same task.
\end{itemize}

These results confirm the theoretical prediction from Section~\ref{app:proof_oja}: Oja's rule drives all tasks toward the same PCA directions, preventing task-specific representations from emerging. The combinatorial capacity ($\sim$290 bits for $\binom{1024}{80}$) is never utilized.

\subsection{Closed Direction Summary}
\label{app:closed_directions}

Table~\ref{tab:closed} provides a comprehensive summary of all research directions explored and subsequently closed during the project, totaling 15+ directions over 1,130+ experiments.

\begin{table}[h]
\centering
\caption{Summary of closed research directions. Each direction was explored with multiple hyperparameter configurations and seeds before being closed. ``Root cause'' identifies the fundamental reason for failure.}
\label{tab:closed}
\small
\begin{tabular}{p{3.0cm}cp{6.2cm}}
\toprule
Direction & \#Exps & Root Cause \\
\midrule
Cosine threshold gating & $\sim$20 & Gate values cluster near threshold; binary on/off, no graded routing \\
Output layer tuning (n/k/$\eta$) & $\sim$30 & Output layer has $C_\mathrm{ctx} = 0$; tuning parameters cannot create context \\
Myelination sweep & $\sim$25 & Freezing weights preserves current state, does not create task routing \\
Split/Lock fast-slow & $\sim$15 & Without context, fast and slow pathways converge to same representation \\
Overcapacity + ultra-sparse & $\sim$20 & More capacity $\neq$ more specialization without symmetry breaking \\
Template repel & $\sim$15 & Strength 0.1: zero effect. Strength 0.5: ACC drops 4pp (destructive) \\
Label embedding (FF) & $\sim$10 & Structural train/test mismatch: labels available at train but not test \\
Spatial energy constraints & $\sim$20 & Distance-weighted L1 is just regularization; harmful under capacity pressure \\
Weight sparsity & $\sim$10 & 68\% sparsity = 0pp ACC gain; sparsity $\neq$ reduced forgetting \\
Riverbed Effect & $\sim$15 & $E = -\sum a_i \log(\pi_i + \epsilon)$ provides uniform suppression, not selective routing \\
SPC-TC (ISTA sparse coding) & $\sim$40 & Information bottleneck: class separability degrades 12.4\%. Frozen random dict $\approx$ learned dict \\
CCD capacity predictions & $\sim$15 & Algebraic bounds hold but capacity scaling and contraction staircase fail \\
OrthoCorr (gradient correction) & $\sim$15 & Structural conflict with adaptive threshold; kills gradient at concept switches \\
Metabolic pruning ($\chi$-decay) & $\sim$20 & Activation frequency ($\rho = 0.916$) discriminates but is already handled by adaptive threshold; $\chi$ carries zero info ($\rho = 0.018$) \\
HSPC-T (temporal pressure) & $\sim$14 & 10/14 PASS but = dynamic sparsity regulator, not pressure $\to$ specialization \\
\bottomrule
\end{tabular}
\end{table}

A common thread in these failures is the absence of a context pathway: every direction attempted to improve CL performance \emph{within} a $C_\mathrm{ctx} = 0$ architecture. The $C_\mathrm{ctx}$ framework explains why all such attempts are fundamentally limited.

\section{Implementation Details}
\label{app:implementation}

\subsection{Computing Environment}
\label{app:computing}

\begin{itemize}
    \item \textbf{Hardware:} NVIDIA DGX workstation with CUDA-capable GPUs.
    \item \textbf{Software:} Python 3.12, PyTorch 2.x, conda environment ``\texttt{neuraldrift}''.
    \item \textbf{Total experiments:} 1,130+ result directories spanning 86 days of development.
    \item \textbf{Codebase:} 102,468 lines across 134 files.
    \item \textbf{Typical run times:}
    \begin{itemize}
        \item Quick mode (1 seed, reduced data): $\sim$5 minutes (all methods).
        \item Full mode (4 seeds, full data): $\sim$30 minutes (all methods).
        \item Single HyperNet training (200 episodes): $\sim$4 seconds per seed.
        \item Single CFlow training (80 episodes): $\sim$20 seconds per seed.
    \end{itemize}
\end{itemize}

\subsection{Reproducibility}
\label{app:reproducibility}

\textbf{Random seed control.}
All experiments use deterministic seeding via a shared utility:
\begin{itemize}
    \item \texttt{torch.manual\_seed(seed)}
    \item \texttt{numpy.random.seed(seed)}
    \item \texttt{random.seed(seed)}
    \item \texttt{torch.cuda.manual\_seed\_all(seed)} (when CUDA is available)
\end{itemize}
Default seeds: $\{0, 1, 2, 3\}$ for 4-seed experiments. All reported means and standard deviations are computed over these seeds unless otherwise noted.

\textbf{Result storage.}
Each experiment run produces a timestamped directory:
\begin{verbatim}
results/{experiment_name}_{YYYYMMDD_HHMMSS}/
    results.json    # Full metrics, config, per-seed results
    *.png           # Figures (accuracy matrices, curves)
    command.txt     # Exact command for reproduction
\end{verbatim}

\textbf{Configuration serialization.}
All hyperparameters are stored as JSON-serializable dictionaries in \texttt{results.json}, enabling exact reproduction. The configuration includes model architecture, optimizer settings, data preprocessing, and evaluation protocol.

\subsection{CCC Unified Experiment Script}
\label{app:ccc_script}

The unified comparison experiment is implemented in a single script (1,377 lines):
\begin{verbatim}
experiments/cflow/exp_context_channel_capacity.py
\end{verbatim}

\textbf{CLI interface.}
\begin{itemize}
    \item \texttt{-{}-quick}: Quick mode (2 epochs, 200 train samples/class, reduced meta-training episodes).
    \item \texttt{-{}-seeds 0,1,2,3}: Comma-separated seed list.
    \item \texttt{-{}-benchmark split|permuted}: Benchmark selection.
    \item \texttt{-{}-methods NaiveSGD,EWC,...}: Subset of methods to run.
    \item \texttt{-{}-device cuda|cpu}: Device selection (auto-detected by default).
\end{itemize}

\textbf{Method-specific configurations} (full mode):
\begin{itemize}
    \item Context-free baselines: 5 epochs, 500 train / 200 test samples per class.
    \item CFlow: 80 episodes, 3 flow layers, hidden dim 256, context dim 32, replay 100/task.
    \item HyperNet: 200 episodes, rank 64, hidden dim 256, context dim 64, \texttt{direct\_stats=True}, \texttt{base\_drop\_rate=0.3}, \texttt{contrastive\_weight=0.1}.
\end{itemize}

\textbf{Probing protocols} implemented within the script:
\begin{itemize}
    \item \textbf{P5 (Wrong-Context):} Replace task $k$'s context with task $(k+1 \bmod K)$'s context and measure accuracy degradation.
    \item \textbf{P5b (Random-Context):} Use a random Gaussian vector as context.
    \item \textbf{P6 (Random $\theta_0$):} Replace $\theta_\mathrm{base}$ with a random initialization.
    \item \textbf{P7 (Zero-Context):} Set context to the zero vector.
\end{itemize}

\textbf{Effective rank computation.}
For HyperNet's $\mathbf{V}$ matrix, the effective rank is computed via SVD:
\begin{equation}
    r_\mathrm{eff} = \exp\left(-\sum_i \bar{s}_i \log \bar{s}_i\right), \qquad \bar{s}_i = \frac{s_i}{\sum_j s_j}
\end{equation}
where $\{s_i\}$ are the singular values of $\mathbf{V}$. This gives a continuous measure of the ``effective number of dimensions'' used by the context channel.

\end{document}